\documentclass{article}

\usepackage{arxiv}

\usepackage[utf8]{inputenc} 
\usepackage[T1]{fontenc}    
\usepackage{hyperref}       
\usepackage{url}            
\usepackage{booktabs}       
\usepackage{amsfonts}       
\usepackage{nicefrac}       
\usepackage{microtype}      
\usepackage{lipsum}		
\usepackage{graphicx}
\usepackage[numbers]{natbib}
\usepackage{doi}

\usepackage{xcolor}         

\usepackage{times}
\usepackage{soul}
\usepackage{graphicx}
\usepackage{amsmath,amsfonts,amsthm,amssymb} 
\usepackage{pifont}
\usepackage{algorithm}
\usepackage[noend]{algpseudocode}
\usepackage{bbm}
\usepackage{bm}
\usepackage{mathrsfs}
\usepackage{subcaption}
\usepackage{xr}
\usepackage{bbold}
\usepackage{caption}
\usepackage[greek,english]{babel}
\usepackage{teubner}
\usepackage{stmaryrd}
\usepackage{multirow}

\newcommand{\hquad}{\hspace{0.5em}}

\newtheorem{thm}{Theorem}

\newtheorem{definition}{Definition}
\newtheorem{corollary}{Corollary}

\title{Eliminating Label Leakage in Tree-based Vertical Federated Learning}

\author{
  Hideaki Takahashi\thanks{The University of Tokyo (this work was done during the internship at Institute for AI Industry Research, Tsinghua University),\newline \texttt{takahashi-hideaki567@g.ecc.u-tokyo.ac.jp}} \and
  Jingjing Liu\thanks{Institute for AI Industry Research, Tsinghua University, \texttt{jjliu@air.tsinghua.edu.cn}} \and
  Yang Liu\thanks{Corresponding Author, Institute for AI Industry Research, Tsinghua University, Shanghai Artificial Intelligence Laboratory, \texttt{liuy03@air.tsinghua.edu.cn}}}

\renewcommand{\shorttitle}{\textit{arXiv} Template}

\hypersetup{
pdftitle={Eliminating Label Leakage in Tree-based Vertical Federated Learning},
pdfsubject={cs.AI, cs.CR},
pdfauthor={Hideaki~Takhashi, Jingjing~Liu, Yang~Liu},
pdfkeywords={Privacy Preserving Machine Learning, Vertical Federated Learning, Label Leakage},
}

\begin{document}
\maketitle

\begin{abstract}
Vertical federated learning (VFL) enables multiple parties with disjoint features of a common user set to train a machine learning model without sharing their private data. Tree-based models have become prevalent in VFL due to their interpretability and efficiency. However, the vulnerability of tree-based VFL has not been sufficiently investigated. In this study, we first introduce a novel label inference attack, \textbf{ID2Graph}, which utilizes the sets of record-IDs assigned to each node (i.e., instance space) to deduce private training labels. ID2Graph attack generates a graph structure from training samples, extracts \textit{communities} from the graph, and clusters the local dataset using community information. To counteract label leakage from the instance space, we propose two effective defense mechanisms, \textbf{Grafting-LDP}, which improves the utility of label differential privacy with post-processing, and \textbf{ID-LMID}, which focuses on mutual information regularization. Comprehensive experiments on various datasets reveal that ID2Graph presents significant risks to tree-based models such as Random Forest and XGBoost. Further evaluations of these benchmarks demonstrate that our defense methods effectively mitigate label leakage in such instances.
\end{abstract}


\section{Introduction}

Tree-based models, including a single decision tree and tree ensembles, such as random forests (RFs) and gradient-boosted decision trees (GBDTs), are among the most widely utilized machine learning algorithms in practice~\cite{kaggle2021}. These models work by recursively partitioning the feature space and training a decision tree to make predictions based on these partitions. When training a decision tree, each internal node splits the \textit{instance space}, the set of training sample IDs assigned to this node, according to a specific discrete function of the input attributes. Due to growing privacy concerns, tree-based vertical federated learning (T-VFL) has grown in popularity~\cite{cheng2021secureboost,wu2020privacy,yin2021comprehensivesurvey,li2022opboost}. T-VFL enables multiple parties with disjoint features of a common user set to train a global tree model collaboratively~\cite{Liu2022VFLsurvey} without exposing their original data. One example is a medical diagnosis model trained on datasets from several hospitals, in which case each hospital may possess a different set of features extracted from the same patient~\cite{liu2019communication}. Another example is financial companies who prefer to use each other's private features to create a credit scoring model~\cite{yang2019federated,zheng2020vertical}. In a typical VFL, only one party, referred to as the active party, has training labels, whereas we call all the other parties passive parties~\cite{yang2023survey}.

Since labels are considered precious and sensitive assets in many scenarios, assessing the security of existing T-VFL methods is imperative. However, most existing research on label leakage attacks against VFL focuses on neural networks~\cite{fu2022label,sun2022label,li2021label}, while tree-based models are deployed more widely than neural networks in practical applications~\cite{kaggle2021}. Although~\cite{cheng2021secureboost} makes a hypothesis that a passive party might be able to infer training labels from the publically shared instance spaces, which are sets of record-IDs assigned to each node, it has not been well studied how and to what extent the attacker can steal training labels. This concern has not been widely recognized, as many recent works~\cite{cheng2021secureboost,10.1007/11562382_75,liu2020federated,hou2021verifiable,10.1145/3448016.3457241,chen2021secureboost+,XU2023237,9514457,yao2022efficient,tian2020federboost,zhu2021pivodl,wang2022feverless} still reveal the instance space to all parties.

In this study, we propose a novel \textbf{ID2Graph} attack, which allows an honest-but-curious passive party to infer private training labels from instance space exposed with high accuracy. We execute the attack by extracting a graph structure from data records used to train the tree-based model and then applying community detection to cluster the learned graph. To eliminate such leakage risk, we propose two effective defense mechanisms: \textbf{Grafting-LDP} and \textbf{ID-LMID}. Grafting-LDP is based on label differential privacy, and ID-LMID utilizes mutual information regularization. Our contributions are two-fold: $1)$ we propose a novel label inference attack against tree-based vertical federated learning and demonstrate its effectiveness through comprehensive experiments; $2)$ we present scalable defense algorithms and showcase their superiority to existing defense mechanisms.

The rest of the paper is structured as follows. First, we overview the workload of T-VFL in Sec.~\ref{sec:tvfl}. Next, we propose a novel label leakage attack, ID2Graph, in Sec.~\ref{seq:attack} and effective defense algorithms, Grafting-LDP and ID-LMID in Sec.~\ref{sec:defense}. Then, in Sec.~\ref{seq:experiments}, we experimentally evaluate our methods on diverse datasets. Sec.~\ref{sec:rw} overviews the existing studies related to this work. Sec.~\ref{seq:conclusion} concludes with a discussion and directions for future work.

\section{Tree-based Vertical Federated Learning}
\label{sec:tvfl}

\begin{figure}[!th]
    \centering
    \includegraphics[width=\linewidth]{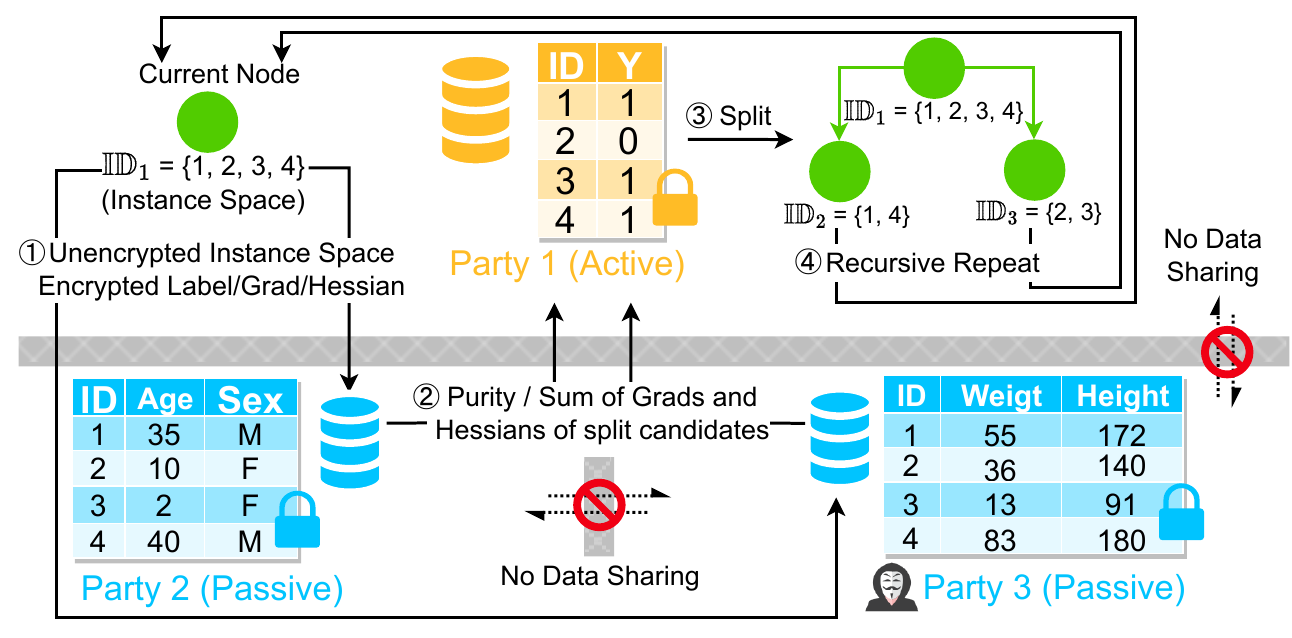}
    \caption{Overview of a tree-based VFL framework where three parties collaboratively train a tree-based model by securely evaluating each split candidate. Since the active party shares the instance space of a node for further split, one malicious passive party might be able to steal ground-truth label information from the instance space the active party shared with him.
    }
    \label{fig:problemdefinition}
\end{figure}

This section first summarizes the tree-based models and then overviews the typical workflow of T-VFL. 

Tree-based methods learn a model by recursively splitting the instance space of the node into several subspaces with certain criteria. For example, Random Forest determines the best split with purity-based gain for the classification task from the set of sub-sampled features. One of the popular metrics is Gini-gain defined as $\frac{n^{L}}{n} \sum_{c \in \mathbb{C}} (\frac{n^{L,c}}{n^{L}})^2 + \frac{n^{R}}{n} \sum_{c \in \mathbb{C}} (\frac{n^{R,c}}{n^{R}})^2 - \sum_{c \in \mathbb{C}} (\frac{n^{c}}{n})^2$, where $n$, $n^{L}$ and $n^{R}$ are the number of samples assigned to the parent, left child, and right child node, respectively. $n^{L,c}$ and $n^{R,c}$ are the numbers of data points with class $c$ within the left and right nodes. XGBoost~\cite{chen2016xgboost} chooses the best split threshold that maximizes the following gain: $\frac{1}{2}[\frac{(g^L)^{2}}{h^{L} + \lambda} + \frac{(g^{R})^{2}}{h^{R} + \lambda} - \frac{g^{2}}{h + \lambda}] - \gamma$, where $g$ and $h$ are the sums of gradients and hessians within the current node, and $g^{L}$, $g^{R}$, $h^{L}$, and $h^{R}$ are those of left and right child nodes after the split, respectively. $\lambda$ and $\gamma$ are hyper-parameters.
The instance space $\mathbb{ID}^{t}_{u}$ is defined as the assignment of data sample IDs to the $u$-th node of the $t$-th tree. A node is called a leaf when it does not have any children.

In many prior works on T-VFL~\cite{cheng2021secureboost,liu2020federated,10.1145/3448016.3457241,chen2021secureboost+,XU2023237,9514457,yao2022efficient,zhu2021pivodl,wang2022feverless}, active and passive parties train a tree-based model by securely communicating necessary statistics to find the best split as well as plain-text instance space for each split. Specifically, Secureboost~\cite{cheng2021secureboost} repeatedly communicates encrypted gradients and hessian in XGBoost, as well as instance spaces and split information of the current node, including best-split feature ID and threshold ID. Similarly, Vertically Federated Random Forest~\cite{yao2022efficient} repeatedly communicates ciphered purity score, as well as instance spaces and split information of the current node. See Fig~\ref{fig:problemdefinition} for the overview of a T-VFL. 

Algorithms \ref{alg:tvfl:passive} and \ref{alg:tvfl:active} depict the typical workflows to construct a tree in Tree-based Vertical Federated Learning (T-VFL) for the passive and active parties, respectively. Since each T-VFL protocol mentioned in Sections 1 and 6 adopts different methods, we provide only abstract overviews here. We denote $m$-th party as $\mathcal{P}_{m}$, where $\mathcal{P}_{1}$ is the active party that possesses the training labels. We assume that $\mathcal{P}_{m}$ has $|\mathbb{F}_{m}|$ features. In T-VFL, the passive party ($\mathcal{P}_m$) usually receives the instance space of a node to be divided ($\mathbb{ID}$). Then, the passive party iterates through all features and calculates the percentiles of each feature on the instances in $\mathbb{ID}$ (Line 4 in Algo.~\ref{alg:tvfl:passive}). In Algo.~\ref{alg:tvfl:passive}, we denote the value of $i$-th samples' $\phi$-th feature owned by $m$-th party as $x^{m}_{i, \phi}$. Next, the passive party generates a binary split for each feature by comparing each instance's feature value to the selected percentile (Line 4 $\sim$ 7 in Algo.~\ref{alg:tvfl:passive}). The split is evaluated using a scoring function, and the result is stored in $\Psi^{m}$. As the scoring function, Random Forest typically uses gini impurity for classification, and XGBoost adopts its own gain function. Once all splits for all features are evaluated, $\Psi^{m}$ is sent to the active party ($\mathcal{P}_1$). If $\mathcal{P}_m$ is chosen as the best party, it receives the best-split index $k^*$ from $\mathcal{P}_1$ and sends the instance spaces of the two child nodes generated by the best split to $\mathcal{P}_1$ (Line 9 $\sim$ 12 in Algo.~\ref{alg:tvfl:passive}). These procedures are recursively continued until termination conditions, such as depth constraints, are met.

\begin{algorithm}[!th]
\caption{Split Finding for T-VFL (Passive Party)}
\label{alg:tvfl:passive}
\begin{algorithmic}[1]
\State $\mathcal{P}_{m}$ receives the instance space $\mathbb{ID}$ of a node to be divided
\State $\Psi^{m} \leftarrow \{\}$
\For{$\phi = 1 \leftarrow |\mathbb{F_{m}}|$}
    \State $\{s_{\phi, 1} , s_{\phi, 2}, ..., s_{\phi, l}\}$ $\leftarrow$ percentiles on $\{x^{m}_{i, \phi}\ | i \in \mathbb{ID}\}$
    \For{$\upsilon = 1 $ to $l$}
        \State Left instance space $\leftarrow \{i | i \in \mathbb{ID}, x^{m}_{i, \phi} < s_{\phi, l}\}$
        \State Right instance space $\leftarrow \{i | i \in \mathbb{ID}, x^{m}_{i, \phi} \geq s_{\phi, l}\}$
        \State Add the evaluation of this split to $\Psi^{m}$
    \EndFor
\EndFor
\State $\mathcal{P}_{m}$ sends $\Psi^{m} = \{\sigma_{m, k}\}$, the evaluation for each split, to $\mathcal{P}_{1}$
\If{$\mathcal{P}_{m}$ is selected as the best party}
\State $\mathcal{P}_{m}$ receives $k^{*}$ from $\mathcal{P}_{1}$
\State $\mathcal{P}_{m}$ sends the childrens' instance spaces generated by $k^{*}$
\EndIf
\end{algorithmic}
\end{algorithm}

\begin{algorithm}[!th]
\caption{Split Finding for T-VFL (Active Party)}
\label{alg:tvfl:active}
\begin{algorithmic}[1]
\State $\mathcal{P}_{1}$ broadcasts the instance space $\mathbb{ID}$ of a node to be divided
\State $\mathcal{P}_{1}$ gathers the evaluation of split candidates $\Psi = \{\Psi^{m}\}^{M}_{m = 1}$.
\State $\mathcal{P}_{1}$ picks the best split $\sigma_{m^{*}, k^{*}}$ from $\Psi$, and notifies $k^{*}$ to $\mathcal{P}_{m^{*}}$
\State $\mathcal{P}_{1}$ receives the instance space of children nodes from $\mathcal{P}_{m^{*}}$
\State If terminated conditions are not satisfied, the children are recursively trained.
\end{algorithmic}
\end{algorithm}

\begin{figure}[!th]
\centering
\includegraphics[width=\textwidth]{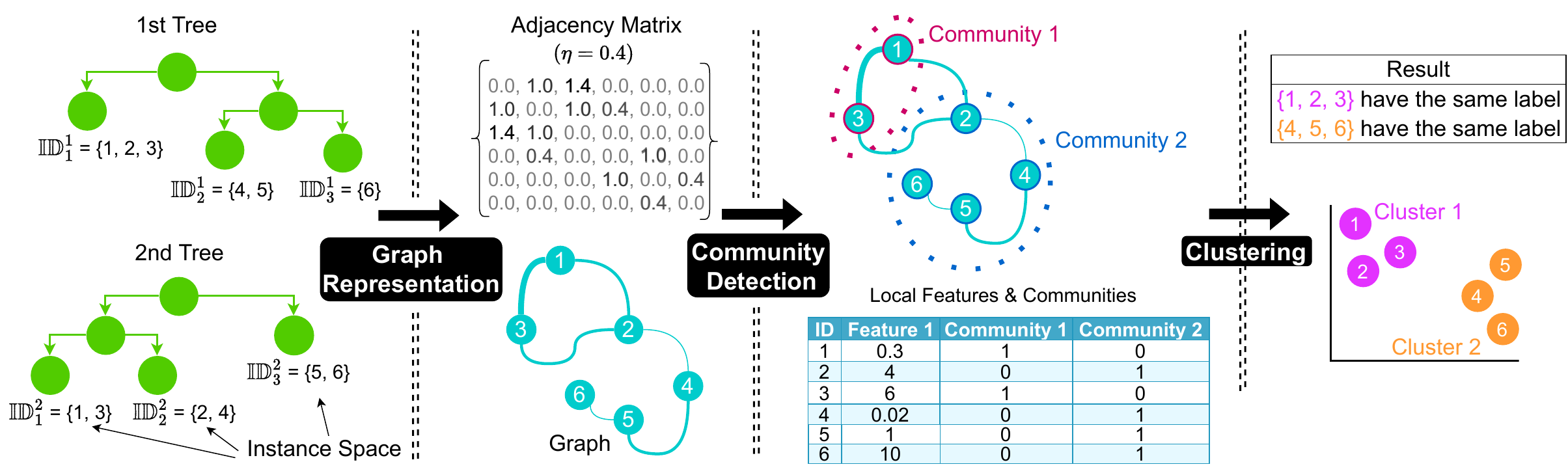}
\caption{Illustration of ID2Graph. This method consists of 1) creating a graph of data points from the trained model, 2) grouping the vertices of the graph to communities, and 3) clustering the local dataset along with extracted communities to estimate which data samples belong to the same label.}
\label{fig:concept}
\end{figure}

During the inference phase, the active party takes charge of predicting new samples. The active party guides the model's path based on decisions and the features possessed by each party until they arrive at a final prediction. The active party has complete knowledge of who has the feature for splitting each node and the weight of each node. When determining whether a given sample should move to the left or right child from the current node, the active party consults the owner of the feature to split that node. It asks this owner to assess whether the new sample is below the threshold or not. Once the new record reaches a leaf, the active party quickly obtains the predicted value for that record in the current tree.

\section{ID2Graph Attack}
\label{seq:attack}

Our proposed attack, ID2Graph, is an attack that a passive party uses to infer training labels $\bm y$ from the instance space exposed to him during T-VFL training. The threat model ID2Graph assumes is compatible with the standard T-VFL settings described in Sec.~\ref{sec:tvfl}. To extract intrinsic patterns of training samples, ID2Graph consists of three stages (See Fig.~\ref{fig:concept}): 1) \textit{Graph Representation}: converting the trained model to a graph of record ids; 2) \textit{Community Detection}: abstracting communities within the graph; 3) \textit{Clustering}: clustering the dataset with assigned communities.

\paragraph{Threat Model}

As the threat model, we assume the entire dataset has $N$ records, and each of $M$ parties has a subset of disjoint features. We denote the set of local datasets as $\{\bm X^{m} \in \mathbb{R}^{N \times |\mathbb{F}_{m}|}\}_{m=1}^{M}$, where $\mathbb{F}_{m}$ is the set of local features of the $m$-th party, and $|\mathbb{F}_{m}|$ is the number of features of the $m$-th party. $\forall m \neq m' \in \{1, ..., M\}, \hquad \mathbb{F}_{m} \cap \mathbb{F}_{m'} = \emptyset$. We focus on classification tasks with $|\mathbb{C}|$ classes, where $\mathbb{C}$ is the set of classes, and only the active party, has the ground-truth labels $\bm y \in \mathbb{R}^{N}$. Let $\mathcal{P}_{m}$ be the $m$-th party, $\mathcal{P}_{1}$ be the active party, and the others be the passive parties. All parties jointly train $|\mathbb{T}|$ trees in total ($\mathbb{T}$ denotes the set of trees).

Based on the T-VFL algorithms above, we consider the honest-but-curious threat model. Specifically, all parties adhere to the given protocol, do not possess auxiliary information, nor engage in side-channel attacks such as analyzing timing, power consumption, or network traffic. The attack controls one passive party $\mathcal{P}_{m}$ where it has access to the local data $\bm X^{m}$ and instance spaces sent to him from the active party $\mathbb{ID}_{m}=\{\mathbb{ID}^{t}_{u}\}_{(t,u) \in \mathcal{P}_{m}}$, and the number of classes $|\mathbb{C}|$. Each party also knows the threshold values only for nodes whose threshold feature is owned by that party. Given the above information, the attacker attempts to infer which training data points belong to the same class.

\paragraph{Step 1: Graph Representation}

As suggested in~\cite{cheng2021secureboost}, it is natural to assume that data instances assigned to the same leaf of the trained model share similarities. Thus, we use a graph structure to represent the relationships across training samples. Each vertex of the graph represents a corresponding data point, and two vertices are connected if they belong to the same leaf (we use \textit{vertex} for graph and \textit{node} for tree). Specifically, we convert a trained tree model to an adjacency matrix $A \in \mathbb{R}^{N \times N}$, as follows:

\begin{equation}
\label{eq:adj}
    A_{i, i'} = \sum_{t=1}^{|\mathbb{T}^{m}|} \sum_{u=1}^{U_t} \eta^{t-1} \mathbb{1}_{\mathbb{ID}^{t}_{u}}(x_{i})\mathbb{1}_{\mathbb{ID}^{t}_{u}} (x_{i'}) \mathbb{1}_{\mathbb{leaf}^{t}}(u)
\end{equation} Here, $|\mathbb{T}^{m}|$ is the number of trees available to the passive party $m$; $U_t$ is the number of nodes of the $t$-th available tree from the viewpoint of the attacker; $\mathbb{leaf}^{t}$ is the set of leaves within the $t$-th tree; and $\eta$ is the discount weight of each tree. $\mathbb{1}$ is the indicator function where $\mathbb{1}_{\mathcal{S}}(x)= 1 \ \ if x \in \mathcal{S}, else \ \ 0$. While the order of leaves does not change $A$, the ordering of trees can influence $A$ when $\eta \neq 1$. This work uses $\eta=1$ for bagging like Random Forest where each tree is independently trained, and $\eta < 1$ for boosting like XGBoost where the information about the labels gradually decreases as the training progresses~\cite{cheng2021secureboost}. Note that the leaves the attacker possesses are not the same as ground-truth leaves of the entire model, as the attacker can obtain the instance space of leaf nodes only if it possesses the attribute that splits the node. Our approach only relies on $\mathbb{T}^{m}$ and not on $\mathbb{T}$.

Since storing the adjacency matrix $A$ requires a space complexity of $O(N^2)$ in the worst case, even with the use of a sparse matrix, we adopt an approximate representation for large datasets that reduces the necessary memory to $O(B N)$, where $B$ is an arbitrary integer. To achieve this, we divide the instance space into multiple chunks of length $B$. We use the same intra-chunk edges within each chunk as in Eq.~\ref{eq:adj}. We also add a few inter-chunk edges to maintain the binding relationship across the instance space. Further details can be found in Algo.~\ref{alg-meaf}. If the size of the instance space is greater than or equal to the chunk size $B$, the algorithm adds inter-chunk edges with a weight of $w'$ between the end of one chunk and the beginning of the next. It then iterates through all pairs of the instance space within each chunk and adds the same edge with a weight of $\eta^{t-1}$ to the adjacency matrix.

\begin{algorithm}[!th]
\caption{Memory Efficient Adjacency Matrix}
\label{alg-meaf}
\begin{algorithmic}[1]

\Require The chunk size $B$, the discount factor $\eta$ and weight of inter-chunk edge $w'$
\Ensure Adjacency matrix $A \in R^{N \times N}$

\State $A \leftarrow$ Zero matrix
\For{$t \leftarrow 1, 2, ..., |\mathbb{T}^{m}|$}
\For{$u \leftarrow 1, 2, ...., U_{t}$}
    \State $\mathbb{ID}^{t}_{u} = \{{i_{1}, i_{2}, ..., i_{|\mathbb{ID}^{t}_{u}|}}\}$
    \Comment{Instance space of $u$-th node}
    \State AddEdges($u$, $t$, $\mathbb{ID}^{t}_{u}$)
\EndFor
\EndFor
\State
\Return $A$
\State
\Function{AddEdges}{$u$, $t$, $\mathbb{ID}^{t}_{u}$}
\If{$u$-th node within $t$-th tree is not a leaf}
    \Return
\EndIf
\If{$|\mathbb{ID}^{t}_{u}| < B$}
\For{$j \gets$ 1 to $|\mathbb{ID}^{t}_{u}|$}
\For{$k \gets$ $j + 1$ to $|\mathbb{ID}^{t}_{u}|$}
\State $A_{i_{j}, i_{k}} \leftarrow A_{i_{j}, i_{k}} + \eta^{t-1}$
\State $A_{i_{k}, i_{j}} \leftarrow A_{i_{k}, i_{j}} + \eta^{t-1}$
\Comment{Intra-chunk edges}
\EndFor
\EndFor
\Else
\State $s \gets 0, \quad t \gets 0$
\While{$s \leq |\mathbb{ID}^{t}_{u}|$}
\If{$s \neq 0$}
\State $A_{i_{s}, i_{t}} \leftarrow A_{i_{s}, i_{t}} + w'$
\State $ A_{i_{t}, i_{s}} \leftarrow A_{i_{t}, i_{s}} + w'$
\Comment{Inter-chunk edge}
\EndIf
\State $t \gets \min{(s + B, |\mathbb{ID}^{t}_{u}| + 1)}$
\For{$j \gets s \hbox{ to } t - 1$}
\For{$k = j + 1 \hbox{ to } t - 1$}
\State $A_{i_{j}, i_{k}} \leftarrow A_{i_{j}, i_{k}} + \eta^{t-1} $
\State $ A_{i_{k}, i_{j}} \leftarrow A_{i_{k}, i_{j}} + \eta^{t-1}$
\Comment{Intra-chunk edges}
\EndFor
\EndFor
\State $s \gets t$
\EndWhile
\EndIf
\EndFunction
\end{algorithmic}
\end{algorithm}

\paragraph{Step 2: Community Detection}

The community detection method is a popular tool to cluster vertices of a graph into multiple communities, where vertices in each community are tightly linked, and vertices in different communities are loosely connected. We adopt Louvain method~\cite{blondel2008fast}, one of the fastest community detection methods, which extracts communities by optimizing modularity $Q$ (a metric of community quality) as follows:

\begin{align}
\label{eq:Q}
\begin{split}
    Q = &(1/(2\sum_{v, v'} A_{v, v'})) \sum_{v, v'} [A_{v, v'} - \\ &(\sum_{v'}A_{v, v'} \sum_{v}A_{v, v'})/(2  \sum_{v} \sum_{v'} A_{v, v'})] \delta(\pi_{v}, \pi_{v'})
\end{split}
\end{align}, where $\pi_{v}$ is the community assigned to the $v$-th vertex, and $\delta$ is the Kronecker delta function, i.e., $\delta(\pi_{v}, \pi_{v'})$ is one when $\pi_{v} = \pi_{v'}$ and zero otherwise. Higher $Q$ indicates denser connections within a community and looser links between different communities. After initially assigning each vertex to its own community, Louvain method iteratively executes modularity optimization and community aggregation to maximize $Q$. During the modularity optimization phase, Louvain method moves each vertex to the best neighboring community, which improves $Q$ until saturation. Then, Louvain method generates a new graph whose vertices represent communities detected during the previous optimization phase.

\paragraph{Step 3: Clustering}

After partitioning the graph of $A$ to the communities, the attacker utilizes allocations of communities as features for clustering. In this study, we apply K-means~\cite{hartigan1979algorithm,ahmed2020k} on the block matrix $[\bar{\bm{X}^{m}}, \alpha \Omega]$ to group the samples to $|\mathbb{K}|$ clusters, where $\bar{\bm{X}^{m}}$ is the min-max normalized dataset of the attacker, $\Omega$ is the dummy variables of assigned communities where $\Omega_{i, j}$ is 1 if $i$-th sample belongs to $j$-th community, $\alpha$ is the weight for $\Omega$ and $\mathbb{K}$ represents the set of cluster labels (see Appendix.~D for the pseudo-code). Since we assume that the attacker knows the number of class categories, the attacker sets the number of clusters $|\mathbb{K}|$ to $|\mathbb{C}|$.


\section{Defense}
\label{sec:defense}

To effectively mitigate label leakage, we develop two innovative defense mechanisms with theoretical guarantees: Grafting-LDP and ID-LMID. Grafting-LDP is founded on differential privacy~\cite{dwork2006differential}, while ID-LMID is grounded in mutual information regularization~\cite{wang2021improving}. We also compare the characteristics of these two methods.

\subsection{Grafting-LDP}

\begin{figure}[!th]
    \includegraphics[width=\linewidth]{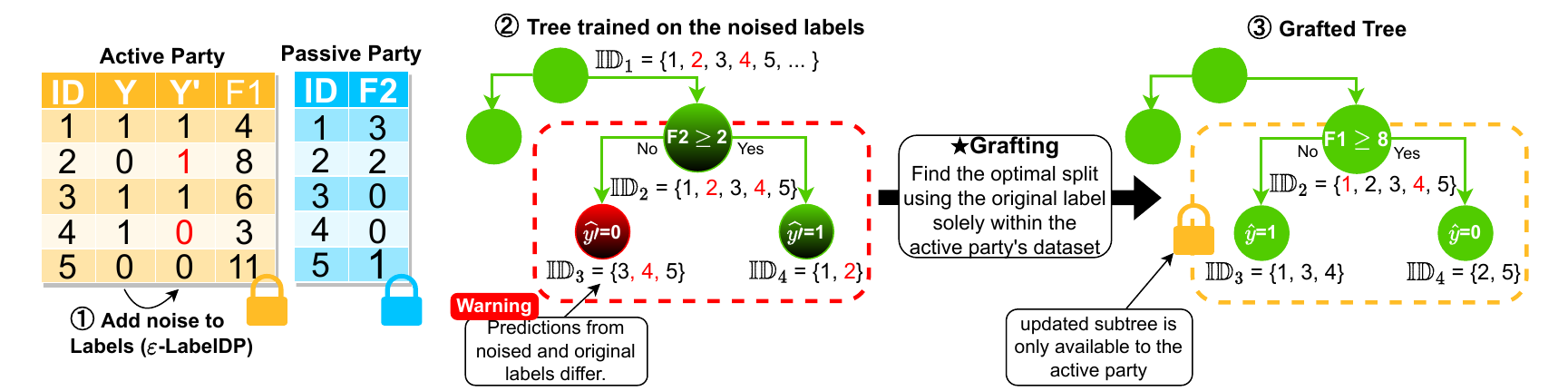}
    \caption{Intuition behind Grafting-LDP. The active party uses noisy labels for collaborative learning so that sharing instance space meets $\epsilon$-LabelDP, resulting in the mismatch of predictions based on the original labels and noisy labels.}
    \label{fig:grafting}
\end{figure}

Our initial defense mechanism, Grafting-LDP, addresses the performance degradation of a tree trained on labels noised with the principles of differential privacy. This is achieved by \textit{grafting}, which means incorporating a new subtree trained on the original clean labels only with the active party's dataset.

Differential privacy (DP) is a popular method that rigorously quantifies information leakage from statistical algorithms~\cite{dwork2014algorithmic,alvim2011differential}. Extending DP, ~\cite{ghazi2021deep,beimel2013private,pmlr-v19-chaudhuri11a} consider the situation where only labels are sensitive information that should be protected. For example,~\cite{ghazi2021deep} defines label differential privacy (LabelDP) as follows:
\begin{definition}[Label differential privacy (LabelDP)]
    Let $\epsilon \in \mathbb{R}_{\geq 0}$, and $\mathcal{M}: \mathscr{D} \rightarrow \mathscr{O}$ be a randomized algorithm. We say that $\mathcal{M}$ is $\epsilon$-label differentially private if for any two datasets $D, D' \in \mathscr{D}$ that differ in the label of a data instance, and any $O \subseteq \mathscr{O}$, we have $\hbox{Pr}[\mathcal{M}(D) \in O] \leq e^{\epsilon} \hbox{Pr}[\mathcal{M}(D') \in O]$
\end{definition} \noindent

Subsequently, there are several methods like LP-MST~\cite{ghazi2021deep} to add noise to training labels so that the training process and the trained model on these noisy labels satisfy $\epsilon$-LabelDP. Thus, if the active party prepares the noisy training labels in advance and uses them for collaborative learning instead of the original labels, sharing the instance spaces of the model trained on those noisy labels also guarantees $\epsilon$-LabelDP, which ensures the indistinguishability of each individual label.

However, we have observed significant performance degradation when using noisy labels. This is primarily due to the fact that as the depth of the tree increases, the number of data samples assigned to a leaf decreases, amplifying the influence of noise within the noisy labels. \textcircled{\scriptsize 1} and \textcircled{\scriptsize 2} of Fig.~\ref{fig:grafting} provide an intuitive example illustrating this phenomenon. Here, the active party owns the true label Y, noised label Y', and the feature F1, while the passive party possesses another feature F2. After the joint training on the noised label, the majority of samples within the red node of the upper tree are noised. Then, the trained tree incorrectly classifies samples with F2 values lower than 2 as belonging to the negative class, although they should be classified as negative.


To address these erroneous decision paths, we propose Grafting-LDP, a post-processing algorithm applicable to bagging-based models such as Random Forest. In essence, Grafting-LDP comprises two phases: the standard collaborative training on noisy labels and post-processing to rectify the model using clean labels, all done on the active party's side (see Algo.~\ref{fig:grafting}). \textcircled{\scriptsize 3} of Fig.~\ref{fig:grafting} shows an example of this process, where the problematic split is replaced with the feature of the active party based on the original clean label so that the fixed tree appropriately fits the clean labels.

In Grafting-LDP, all parties first collaboratively train a bagging-based model using standard T-VFL schemes, where each tree is independently trained (Line 1 $\sim$ 3 in Algo.~\ref{fig:grafting}). Subsequently, the active party starts \textit{grafting} (Line 4 $\sim$ 23 in Algo.~\ref{fig:grafting}), inspecting each tree's nodes with postorder and fixing them. Specifically, it attaches a flag named "IsContam" to each node by executing a subroutine called "CheckContam," which assesses whether a node is overly contaminated due to the presence of noisy labels. IsContam of each leaf node is set to true when the majority category of samples assigned to that node, as calculated using noisy labels, does not match that of the original clean labels. For leaf nodes, the "IsContam" flag is set to true if the majority category of samples assigned to that node, based on the noisy labels, does not align with the category indicated by the original clean labels. For non-leaf nodes, the active party checks whether either the left or right child nodes have an "IsContam" flag set to true. If at least one of them does, the active party proceeds to run "CheckContam" on the current node. If the result is true, the "IsContam" flag of that node is set to true. Otherwise, the active party resorts to re-splitting that node, utilizing its own dataset along with the original clean labels.


\begin{algorithm}[!th]
\caption{Grafting-LDP}
\label{alg:Grafting-LDP}
\begin{algorithmic}[1]

\State \textbackslash* Training Phase *\textbackslash
\State $\mathcal{P}_1$ adds noise to the training labels with the given mechanism satisfying $\epsilon$-LabelDP.
\State All parties jointly train a model on the noisy labels.

\State 	\textbackslash* Grafting Phase *\textbackslash
\For{each Tree}
    \State Grafting(Tree.RootNode)
\EndFor

\State
\Function {Grafting}{Node}
\If{Node is a leaf}
    \State Node.IsContam = CheckContam(Node)
\Else
    \State Grafting(Node.Left)
    \State Grafting(Node.Right)
    \If{Node.Left.IsContam or Node.Right.IsContam}
        \If{ChechContam(Node)}
            \State Node.IsContam = True
        \Else
    \State $\mathcal{P}_1$ erases all children of Node
    \State $\mathcal{P}_1$ splits Node only with its local features on \newline \hspace*{5.7em} the original labels
        \EndIf
    \EndIf
\EndIf
\EndFunction

\Function{CheckContam}{Node}
    \State $\hat{y'} \leftarrow$ Majority category within Node on noisy labels
    \State $\hat{y} \leftarrow$ Majority category within Node on clean labels
    \State
    \Return $\hat{y'} == \hat{y}$
\EndFunction

\end{algorithmic}
\end{algorithm}

Even during the inference phase, the active party does not need to provide any additional information to the passive parties. The active party can determine whether the new sample should move to the right or left at each node of a repaired tree, even when the best feature is owned by a passive party. This determination is made by simply executing inference on the original tree \textit{prior} to its repair.


In summary, Grafting-LDP fixes the trained model by only utilizing the local features and clean labels the active party owns. Since Grafting-LDP does not need the cooperation of other passive parties, it offers a strong security guarantee, as proved in the following theorem.

\begin{thm}[Security of Grafting-LDP]
Let $\mathcal{G}: \mathscr{D} \rightarrow \mathscr{O}$ be the Grafting-LDP algorithm from the perspective of passive parties, where $O \subseteq \mathscr{O}$ be the set of information accessible during the training and the inference phase. Then, $\mathcal{G}$ satisfies $\epsilon$-LabelDP while any party cannot gain any additional information about the datasets of others.
\end{thm}

\begin{proof}
Recall that all information obtainable for passive parties during the training and the model utilized by passive parties during the inference are generated from the dataset with the noisy labels satisfying $\epsilon$-LabelDP. Then, the post-processing property ensures that $\mathcal{G}$ meets $\epsilon$-LabelDP. In addition, since the \textit{grafting} phase does not require the involvement of passive parties, none of the parties can obtain extra information about the datasets of other parties.
\end{proof}

\subsection{ID-LMID}

We also develop another defense named ID-LMID, based on mutual information regularization~\cite{DBLP:journals/corr/abs-2009-05241}. The existing studies on mutual information regularization~\cite{DBLP:journals/corr/abs-2009-05241,wang2016relation,cuff2016differential} assume that privacy is preserved when mutual information (MI) between the sensitive information and the knowledge accessible to the adversary is minimized or lower than the specified threshold. Then, our ID-LMID prevents label leakage by reducing the mutual information between the label and instance space. Since MI directly measures the amount of label information extractable from the instance space, restricting MI leads to less data leakage. 


We first prove we can track the upper bound of mutual information between label and instance space.

\begin{thm}
\label{thm:ourdefense}
Let $X$, $Y$ be the training data and label, respectively, and $S_{w}$ be the indicator variable for the instance space of the $w$-th node of a tree model trained with $X$ and $Y$, that is, $S_{w}=\mathbb{1}_{\mathbb{ID}_{w}}(X)$, where $\mathbb{1}_{\mathbb{ID}_{w}}$ is the indicator function for $\mathbb{ID}_{w}$, the instance space of the $w$-th node. Then, $I(Y;S_{w})$, mutual information between $Y$ and $S_w$, is bounded as follows:

\begin{equation}
\begin{split}
    \label{eq:mi:upperbound}
    I(Y;S_{w}) &\leq \max( \sum_{c \in \mathbb{C}} \frac{n_{w}^{c}}{n_{w}} \log{ \frac{ n_{w}^{c} / n_{w}}{ N^{c} / N}}, \hquad
    \sum_{c \in \mathbb{C}} \frac{\overline{n}^{c}_{w}}{\overline{n}_{w}} \log{ \frac{ \overline{n}^{c}_{w} / \overline{n}_{w}}{N^{c} / N}})
\end{split}
\end{equation}, where $N^{c}$ is the total number of samples in class $c$, $n_{w}$ is the number of samples within the $w$-th node, $n_{w}^{c}$ is the number of samples within the $w$-th node with class $c$, $\overline{n}_{w} = N - n_{w}$, and $\overline{n}^{c}_{w} = N^{c} - n_{w}^{c}$.
\end{thm}

\begin{proof}
    Since by definition $I(Y;S_{w}) =$ $E_{S_{w}}[D_{KL}$ $(P(Y|S_{w})$ $||$ $ P(Y))]$, where $D_{KL}$ is KL-Divergence, and $E$ is expected value, we have the upper bound as follows:

    \begin{equation}
    \begin{split}
    \label{eq:mi_y_s}
        I(Y;S_{w}) &= E_{S_{w}}[D_{KL}(P(Y|S_{w}) \hquad || \hquad P(Y))] \\
        &\leq \max_{s \in \{0, 1\}}(D_{KL}(P(Y|S_{w}=s) \hquad || \hquad P(Y))) \\
        &= \max_{s \in \{0, 1\}}(\sum_{c \in \mathbb{C}}
        P_{Y|S_{w}=s}(c) \log{\frac{P_{Y|S_{w}=s}(c)}{P_{Y}(c)}})
    \end{split}
    \end{equation} Recall that $S_{w}$ is 1 when the data belongs to the node, and 0 otherwise. Since $P(Y)$ is the label distribution, $P(Y|S_{w}=1)$, and $P(Y|S_{w}=0)$ are the label distribution within the $w$-th node and outside the $w$-th node, respectively, we can empirically approximate these terms as follows:

    \begin{align}
    \label{eq:pms}
        P_{Y}(c) = N_c / N, \hquad
        P_{Y|S_{w}=1}(c) = n^{w}_c / n^{w}, \hquad
        P_{Y|S_{w}=0}(c) = \overline{n}^{w}_{c} / \overline{n}^{w}
    \end{align}  Combining Eq.~\ref{eq:mi_y_s} and Eq.~\ref{eq:pms} yields Eq.~3.
\end{proof}




Based on Theorem~\ref{thm:ourdefense}, if the active party aims to control $I(Y;S_{w})$ not to exceed an arbitrary value $\xi$, it can achieve this goal by making any node visible to passive parties satisfy the following condition:

\begin{equation}
\label{eq:lmir:cond}
\begin{split}
\max( &\hbox{$\sum_{c \in \mathbb{C}}$} (n_{w}^c/n_{w}) \log{ (n_{w}^{c} N / n_{w} N^{c})}, \\
    &\hbox{$\sum_{c \in \mathbb{C}}$} (\overline{n}^{c}_{w} / \overline{n}_{w}) \log{( \overline{n}^{c}_{w} N / \overline{n}_{w} N^{c})}
    ) \leq \xi
\end{split}
\end{equation}

We also show that with this threshold set, the passive party cannot learn more information.

\begin{corollary}
   Eq.~\ref{eq:lmir:cond} guarantees that the attacker cannot get more label information than threshold $\xi$ by applying any mechanism $\mathcal{M}$ to the instance space.
\end{corollary}

\begin{proof}
     Based on Theorem~\ref{thm:ourdefense} and following data processing inequality~\cite{cover1999elements}, the output of any mechanism $\mathcal{M}$ that takes $S_{w}$ cannot increase mutual information:
\begin{equation}
\label{eq:composition_lmid}
    I(Y;\mathcal{M}(S_{w})) \leq I(Y;S_{w}) \leq \xi
\end{equation}
\end{proof}

\begin{minipage}[!th]{0.49\textwidth}
\begin{figure}[H]
    \includegraphics[width=\linewidth]{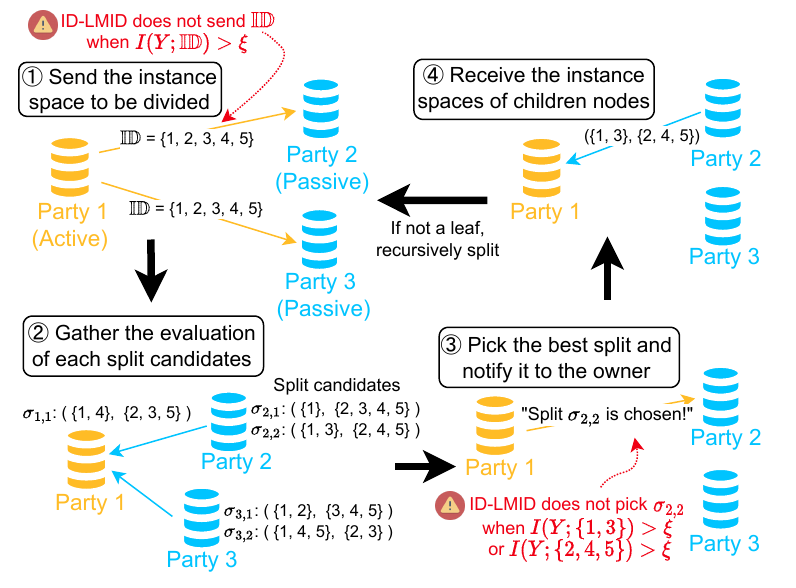}
    \caption{Example of Split Finding within ID-LMID, where ID-LMID identifies potential leakage in instance spaces from both the active and passive sides, and prevents passive parties from accessing the instance spaces that break Eq.~\ref{eq:lmir:cond}.}
    \label{fig:tree-vul}
\end{figure}
\end{minipage}
\hfill
\begin{minipage}[!th]{0.50\textwidth}
\begin{algorithm}[H]
\caption{Split Finding with ID-LMID}
\label{alg:idmlid}
\begin{algorithmic}[1]

\State $\mathcal{P}_{1}$ gathers the evaluation of split candidates $\Psi = \bigcup^{M}_{m=1} \Psi^{m}$.
\For{each split candidate $\sigma_{m, k} \in \Psi$}
\If{$m \geq 2$ and left or right child produced \par\hskip\algorithmicindent by $\sigma_{m, k}$ does not
satisfy Eq.~\ref{eq:lmir:cond}}
\State $\Psi \leftarrow \Psi \setminus \{ \sigma_{m, k} \}$
\EndIf
\EndFor

\State $\mathcal{P}_{1}$ picks the best split $\sigma_{m^{*}, k^{*}}$ from $\Psi$.

\If{left or right child produced by $\sigma_{m^{*}, k^{*}}$ does not satisfy Eq.~\ref{eq:lmir:cond}}
\State If not terminated, all children are trained \newline \hspace*{1.4em} only with the active party $\mathcal{P}_{1}$.
\Else
\State $\mathcal{P}_{m^{*}}$ receives $k^{*}$ and sends the instance  \newline \hspace*{1.4em}  spaces of children nodes to $\mathcal{P}_{1}$.
\State If not terminated, the children are \newline \hspace*{1.4em} recursively trained with all parties.
\EndIf
\end{algorithmic}
\end{algorithm}
\end{minipage}

Then, we propose ID-Label Mutual Information-based Defense (ID-LMID), which makes all instance spaces visible to passive parties satisfy $I(Y;S_{w}) \leq \xi$ under Theorem~\ref{thm:ourdefense}.
Note that under a T-VFL scheme, we observe that a passive party can obtain the instance space of a node under two conditions:  1) it knows the threshold that produces that node, or 2) it directly receives the instance space from the active party.
Based on the above observations, ID-LMID introduces two MI constraints to the original T-VFL protocol when finding the best split for each node. Specifically, to avoid leakage under condition 1),  the active party does not adopt any split candidate from a passive party that generates a child node violating Eq.~\ref{eq:lmir:cond}, which eliminates all of the unsatisfactory split candidates from passive parties. To avoid leakage from condition 2), the active party does not broadcast the instance space when the left or right node of the split breaks Eq.~\ref{eq:lmir:cond}, but searches the best split exclusively using its own dataset.

Algo.~\ref{alg:idmlid} presents the overview of ID-LMID, where $\sigma_{m, k}$ is $k$-th split candidate of $m$-th party, and $\Psi^{m}$ denotes the set of split candidates of $m$-th party. Line 2 $\sim$ 4 implements the first constraint, and Line 6 $\sim$ 10 corresponds to the second constraint. Finally, evaluating the first constraint (Line 4 in Algo.~\ref{alg:idmlid}) requires the active and passive party to calculate node purity ($n^{c}_{w} / n_{w}$ and $\bar{n}^{c}_{w} / \bar{n}_{w}$) in Eq.~\ref{eq:lmir:cond} in a secure manner without exposing sensitive information to each other, which has been widely studied before using multi-party computation (MPC)~\cite{10.1007/11562382_75,9892321}, homomorphic encryption(HE)~\cite{wu2020privacy,liu2020federated,XU2023237,yao2022efficient}, or without protection~\cite{hou2021verifiable,9514457}. This work adopts the Homomorphic Encryption(HE)-based implementation ~\cite{XU2023237,yao2022efficient} and proposes our algorithm for evaluating Eq.~\ref{eq:lmir:cond} (See Algo~\ref{alg:secure-lp}). In Algo~\ref{alg:secure-lp}, we use $y^{c}_{i}$ to denote the $c$-th position of the one-hot encoded label $y_i$, and use $\llbracket \cdot \rrbracket$ to denote a value encrypted with Paillier Encryption~\cite{paillier1999public}, which is popular HE technique that allows the addition between ciphertexts and multiplication between ciphertext and plaintext. Then, the node purity of each split candidate can be securely calculated by summing the corresponding encrypted labels. Similar to many existing frameworks~\cite{cheng2021secureboost,chen2021secureboost+,XU2023237,yao2022efficient}, this procedure only discloses the aggregated statistics to the active party, and no other information is revealed to any party.

\begin{algorithm}[!th]
\caption{Seure Computation of Node Purity with HE}
\label{alg:secure-lp}
\begin{algorithmic}[1]

\State $\mathcal{P}_{1}$ encrypts one-hot encoded label $\{\{y^{c}_{i}\}^{C}_{c=1}\}^{N}_{i=1}$ with Paillier Encryption and broadcasts $\{\{\llbracket y^{c}_{i} \rrbracket \}^{C}_{c=1}\}^{N}_{i=1}$ to all passive parties.

\For{$m \leftarrow 2 ... M$}
\For{$\sigma_{m,k} \in \Psi^{m}$}
    \State $\mathbb{ID}_{L}, \hquad \mathbb{ID}_{R} \leftarrow$ Instance space of the left and right child nodes divided with $\sigma_{m,k}$
    \State $\overline{\mathbb{ID}_{L}} \leftarrow \{1, 2, ..., N\} \setminus \mathbb{ID}_{L}, \hquad \overline{\mathbb{ID}_{R}} \leftarrow \{1, 2, ..., N\} \setminus \mathbb{ID}_{R}$
    \State $\mathcal{P}_{m}$ sends $\{ (1 / |\mathbb{ID}_{L}|) \sum_{i \in \mathbb{ID}_{L}} \llbracket y^{c}_{i} \rrbracket \}^{C}_{c=1}$ and $\{ (1 / |\overline{\mathbb{ID}_{L}}|) \sum_{i \in \overline{\mathbb{ID}_{L}}} \llbracket y^{c}_{i} \rrbracket \}^{C}_{c=1}$ to $\mathcal{P}_{1}$
    \State $\mathcal{P}_{m}$ sends $\{ (1 / |\mathbb{ID}_{R}|) \sum_{i \in \mathbb{ID}_{R}} \llbracket y^{c}_{i} \rrbracket \}^{C}_{c=1}$ and $\{ (1 / |\overline{\mathbb{ID}_{R}}|) \sum_{i \in \overline{\mathbb{ID}_{R}}} \llbracket y^{c}_{i} \rrbracket \}^{C}_{c=1}$ to $\mathcal{P}_{1}$
\EndFor
\State $\mathcal{P}_{1}$ decrypts purities submitted by $\mathcal{P}_{m}$ for both children and evaluates Eq.~\ref{eq:lmir:cond}.
\EndFor
\end{algorithmic}
\end{algorithm}

\subsection{Grafting-LDP vs ID-LMID}
\label{subseq:defense-vs}

On the one hand, the advantage of Grafting-LDP is that the training procedure with the passive parties guarantees $\epsilon$-LDP, which is the rigorous and well-studied notion of privacy while not requiring additional communication. Its implementation is simple, as shown in Algo.~\ref{alg:Grafting-LDP}, and the practitioners can easily apply Grafting-LDP to the existing bagging-based T-VFL. However, applying Grafting-LDP to boosting methods like XGBoost is inappropriate since each tree is not independently trained. For example, in XGBoost, the $i$-th tree is trained to fit the residual between the ground truth and the prediction of the prior $i-1$ trees.  If we repair the $i-1$-th tree with grafting, the residual, which is the objective for the $i$-th tree to fit, changes so that we have to re-train the entire $i$-th tree, not its subtree. In other words, if we forcibly apply Grafting-LDP to boosting-based models, we can execute grafting only for the first tree, and the other trees should be re-trained from scratch with only the active party, which is incompatible with VFL.

On the other hand, ID-LMID can be applicable for both boosting and bagging. Although ID-LMID might need additional communication to evaluate the upper bounds of MI, many implementations of Random Forest for T-VFL~\cite{XU2023237,yao2022efficient} already communicate the same HE-encrypted purities for tree training, eliminating the need for additional communication costs. \cite{cheng2021secureboost} also shows that summing up the ciphertext encrypted with Paillier Encryption for each threshold candidate and communicating them can scale for large datasets. In addition, the existing study~\cite{DBLP:journals/corr/abs-2009-05241} reveals that mutual regularization gives a better privacy-utility tradeoff compared to differential privacy, which is compatible with our experiment in Sec.~5.


\section{Experiments}
\label{seq:experiments}

We conduct comprehensive experiments on various datasets to demonstrate the effectiveness of our proposed attack and defense algorithms. We also show various factors that influence the performance of our approaches, including the impact of feature importance, tree depth, and the number of trees.

\paragraph{Datasets and Models}

We conduct experiments on a two-party VFL system over nine different datasets: 1) \textit{Breastcancer}; 2) \textit{Parkinson}; 3) \textit{Obesity}; 4) \textit{Phishing}; 5) \textit{Avika}; 6) \textit{Drive}; 7) \textit{Fmnist}; 8) \textit{Fars} and 9) \textit{Purcio}. We utilize 80\% of each dataset for training and the remaining 20\% of data as the test dataset. All the labels are held by one of two parties, the active party, while the other passive party tries to steal the labels. For all datasets except \textit{Fmnist}, we vertically and randomly partition features into two halves as the local datasets of the two parties. Since \textit{Fmnist} consists of images, we divided each image equally into left and right and gave one side to the active party and the other to the passive party. Note that from the attacker's perspective, this setting is the same as when the attacker possesses 50\% of features and one active and multiple passive parties have the remaining features. Tab.~\ref{tab:datasets} shows the details of each dataset. The number of datasets and their scales are larger than many related studies~\cite{cheng2021secureboost,wu2020privacy,liu2021enabling,li2022opboost,zhu2021pivodl}.

\begin{table}[!th]
\caption{Statistics of Datasets}
\label{tab:datasets}
\centering
\begin{tabular}{|c||c|c|c|}
\toprule
Dataset               & \#Samples & \#Features & \#Classes\\ \midrule
\textit{Breastcancer}~\cite{Dua:2019} & 569               & 30                 & 2 \\
\textit{Parkinson}~\cite{Dua:2019}    & 756               & 754                & 2 \\
\textit{Obesity}~\cite{Dua:2019,palechor2019dataset}      & 2111              & 17                 & 7 \\
\textit{Phishing}~\cite{Dua:2019}     & 11055             & 30                 & 2 \\
\textit{Avila}~\cite{Dua:2019,10.1016/j.engappai.2018.03.023}        & 20867             & 10                 & 12 \\
\textit{Drive}~\cite{Dua:2019}        & 58509             & 49                 & 11 \\
\textit{Fmnist}~\cite{xiao2017/online}       & 60000             & 784                & 10 \\
\textit{Fars}~\cite{Olson2017PMLB}         & 100968            & 30                 & 8 \\
\textit{Pucrio}~\cite{Dua:2019,ugulino2012wearable}       & 165632            & 18                 & 5 \\ \bottomrule
\end{tabular}
\end{table}

\begin{table}[!th]
\scriptsize
\caption{Attack results (CL: clustering, UNI: Union Attack). We measure the performance of each attack with V-measure. ID2Graph leads to better and more stable grouping than baseline on all metrics and models.}
    \label{tab:results:attack}
    \centering
\begin{tabular}{@{}c||c|ccc|ccc@{}}
\toprule
                   &                & \multicolumn{3}{c|}{Random Forest}                         & \multicolumn{3}{c}{XGBoost}                               \\ \midrule
Dataset            & CL             & UNI            & UNI+CL         & ID2Graph                & UNI            & UNI+CL         & ID2Graph                \\ \midrule
\textit{Breastcancer} & 0.554 (±0.084) & 0.000 (±0.000) & 0.554 (±0.084) & \textbf{0.751 (±0.106)} & 0.000 (±0.000) & 0.554 (±0.084) & \textbf{0.736 (±0.133)} \\
\textit{Parkinson} & 0.091 (±0.011) & 0.035 (±0.079) & 0.094 (±0.011) & \textbf{0.349 (±0.100)} & 0.000 (±0.000)     & 0.091 (±0.011) & \textbf{0.224 (±0.111)} \\
\textit{Obesity}   & 0.254 (±0.029) & 0.000 (±0.000) & 0.254 (±0.029) & \textbf{0.610 (±0.079)} & 0.000 (±0.000)     & 0.254 (±0.029) & \textbf{0.549 (±0.114)} \\
\textit{Phishing}     & 0.001 (±0.001) & 0.196 (±0.268) & 0.202 (±0.276) & \textbf{0.352 (±0.156)} & 0.196 (±0.268) & 0.202 (±0.276) & \textbf{0.342 (±0.183)} \\
\textit{Avila}     & 0.085 (±0.057) & 0.043 (±0.095) & 0.107 (±0.054) & \textbf{0.252 (±0.050)} & 0.000 (±0.000)     & 0.085 (±0.057) & \textbf{0.200 (±0.053)} \\
\textit{Drive}     & 0.283 (±0.042) & 0.181 (±0.101) & 0.302 (±0.034) & \textbf{0.729 (±0.032)} & 0.000 (±0.000)     & 0.283 (±0.042) & \textbf{0.660 (±0.045)} \\
\textit{Fmnist}    & 0.431 (±0.002) & 0.133 (±0.182) & 0.432 (±0.002) & \textbf{0.525 (±0.013)} & 0.000 (±0.000)     & 0.431 (±0.002) & \textbf{0.513 (±0.025)} \\
\textit{Fars}      & 0.187 (±0.081) & 0.506 (±0.287) & 0.454 (±0.131) & \textbf{0.564 (±0.056)} & 0.224 (±0.311) & 0.308 (±0.158) & \textbf{0.568 (±0.060)} \\
\textit{Pucrio}       & 0.075 (±0.025) & 0.238 (±0.217) & 0.217 (±0.105) & \textbf{0.555 (±0.035)} & 0.159 (±0.217) & 0.170 (±0.114) & \textbf{0.470 (±0.073)} \\ \bottomrule
\end{tabular}
\end{table}

We employ Random Forest and XGBoost as the target tree-based methods, with a depth of 6, a feature sub-sampling ratio of 0.8, and a tree size of 5 as the default setting. We use a learning rate of 0.3 and cross-entropy loss for XGBoost. Our hyperparameters are consistent with prior works, with the sampling ratio and learning rate from~\cite{cheng2021secureboost}, the depth from~\cite{chen2016xgboost}, and the number of trees based on~\cite{wu2020privacy}. Since our main focus is evaluating potential attack and defense performance rather than achieving high accuracy on the main task of VFL, we did not perform fine-tuning of these hyperparameters.

For evaluation metrics, we use V-measure~\cite{rosenberg2007v}, one of the famous metrics for clustering, to measure how accurately the estimated clusters correspond to ground-truth labels. V-measure varies from 0.0 to 1.0, where 1.0 stands for perfectly accurate clustering. All results are averaged over five different random seeds.

\paragraph{Baselines}

To evaluate our ID2Graph Attack, we compare with the following baselines: 1) Clustering (CL). We apply k-means clustering to the attacker's local features as a baseline for label leakage. 
2) Union Attack (UNI). UNI is the naive approach in Secureboost~\cite{cheng2021secureboost}, where the attacker approximates that two samples have the same label if these two samples are assigned together to at least one node. 3)UNI+CL. This method combines UNI and CL by using the result of UNI as additional features for k-means clustering. To evaluate defenses, we compare them to two existing mechanisms: Reduced-Leakage~\cite{cheng2021secureboost} and LP-MST~\cite{ghazi2021deep}, which protects labels based on randomized response. Since Reduced-Leakage is intended for gradient-boosting, we apply it solely to XGBoost.

\paragraph{Hyper Parameters}

For ID2Graph, we use $\eta$ of 0.6 for XGBoost and 1.0 for Random Forest, and $\alpha$ of 3. We adopt the memory-efficient adjacency matrix for \textit{Drive}, \textit{Fmnist}, \textit{Fars}, and \textit{Pucrio} with a chunk size of 1000 and inter-chunk weight of 100. The stopping criteria of K-means is max iterations of 300 or relative tolerance concerning the Frobenius norm of the difference in cluster centers between two successive iterations of 1e-4, which are the default parameters of \textit{sklearn}~\cite{scikit-learn}, one of the most popular libraries for machine learning. The stopping criteria of the Louvin Method is the max iterations of 100 or tolerance concerning the modularity between two consecutive iterations of 1e-6 to achieve compatible precision with other famous implementations like~\cite{noauthororeditorneo4j,python_louvain}. For defense, we exhaustively search the different privacy budgets to see their privacy-utility trade-offs. Specifically, we set both $\epsilon$ of LP-MST and Grafting-LDP and $\xi$ of ID-LMID to [0.1, 0.5, 1.0, 2.0]. Here, We employ two stages (LP-2ST) for LP-MST, and Grafting-LDP is applied to the Random Forest model trained on LP-MST.

\subsection{Attack Results}

\paragraph{Main Results}

Tab.~\ref{tab:results:attack} summarizes V-measure scores of different attacks. We report the average and the standard deviation of five trials with different random seeds. Results show that ID2Graph leads to higher V-measures than clustering on local features only, implying that the proposed attack can steal private label information from the trained model. ID2Graph also outperforms Union Attack, indicating that ID2Graph can extract more label information from the instance space than Union Attack. Union Attack, in some cases, ends up assigning all samples to the same cluster, which renders the average of V-measure zero. The combination of Union Attack and clustering is still less effective than ID2Graph.





\paragraph{Impact of Feature Partitions}

Additionally, we study how the quality of the attacker's local features impacts the attack performance. For all datasets except \textit{Fmnist}, we use mutual information~\cite{kraskov2004estimating} between each feature and label to quantify feature importance and sort features in descending order. Then, we distribute the top k-percentile features to the attacker and assign the remaining features to the active party. As \textit{Fmnist} is the image dataset, we simply assign the left $k$ percent of each image to the attacker and the rest of the image to the active party. Fig.~\ref{fig:feature} shows the attack results of each method on various percentiles (left y-axis) and feature importance (right y-axis). Note that the left y-axises have different scales for better visibility. ID2Graph surpasses other baselines in most cases, indicating that the extracted community information stabilizes clustering regardless of the constitution of the local dataset. We also notice that the V-measure does not continuously improve for certain datasets as the number of available features increases since including non-informative or bad features may counteract the label inference performance.

\begin{figure}[!th]
    \centering
    \begin{minipage}[b]{0.48\linewidth}
      \centering
      \includegraphics[width=\linewidth]{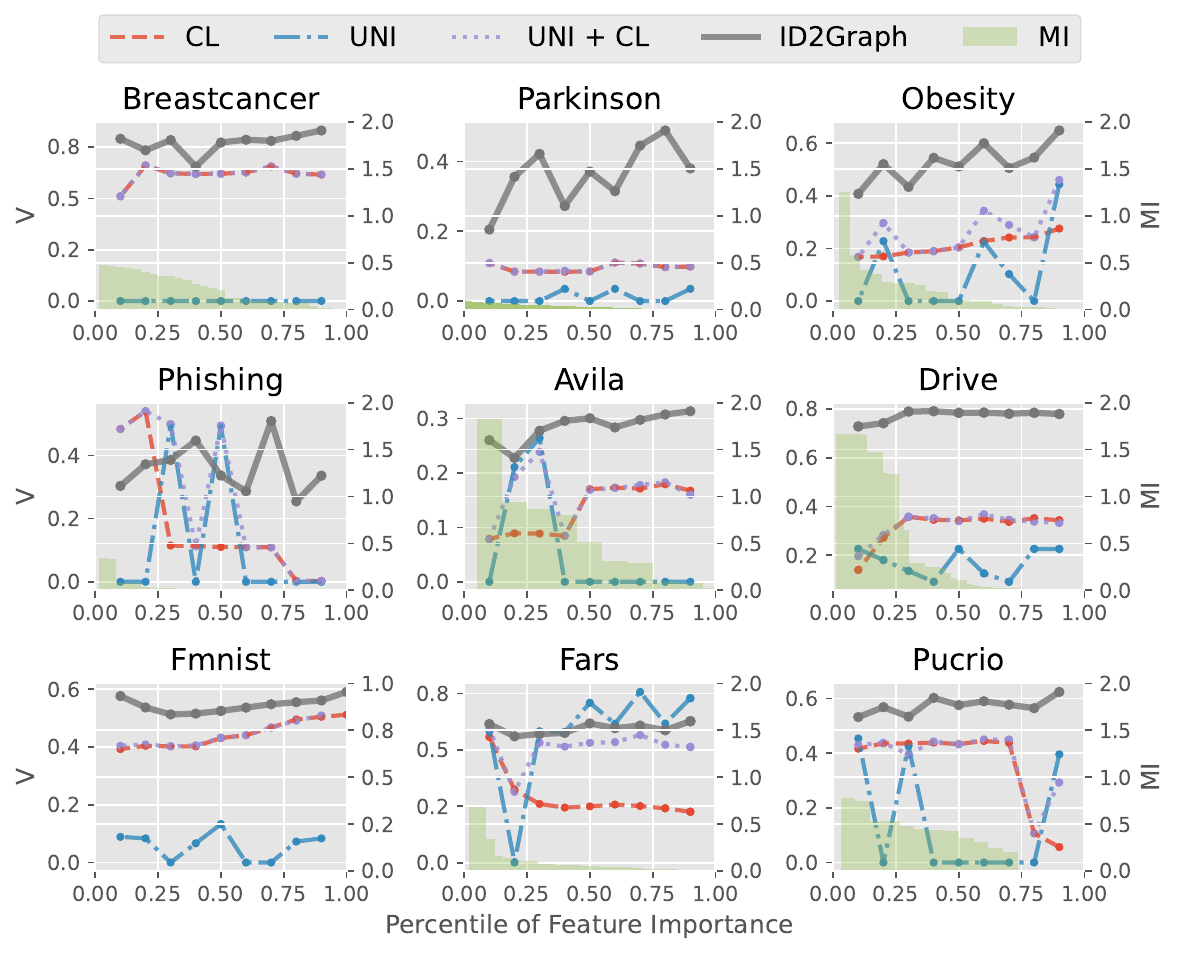}
      \subcaption{Random Forest}
    \end{minipage}
    \hspace{0.002\columnwidth}
    \begin{minipage}[b]{0.48\linewidth}
      \centering
        \includegraphics[width=\linewidth]{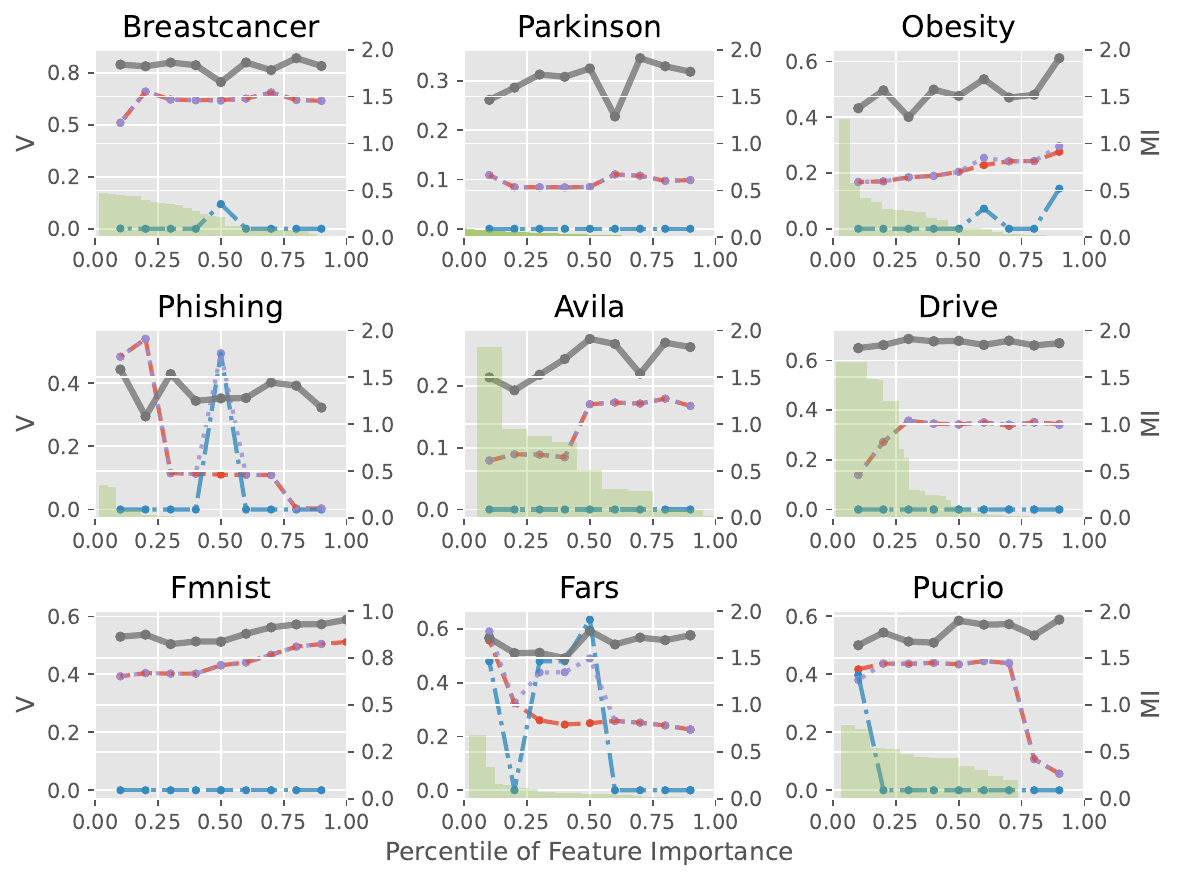}
  \subcaption{XGBoost}
    \end{minipage}
  \caption{Impact of feature partition on attack performance. ID2Graph outperforms baselines in most cases regardless of the informativeness of the local dataset. For all datasets except \textit{Fmnist}, the X-axis is the top \% important features the attacker has. Feature importances based on mutual information scores are represented by bars. For \textit{Fmnist}, the X-axis is the left \% of images the attacker owns. We use different scales for y-axes for better viewing.}
  \label{fig:feature}
\end{figure}

\paragraph{Impact of Tree Depth.}

The performance of ID2Graph is also impacted by the maximum depth constraint of the tree model, as shown in Fig.~\ref{fig:depth}. Increasing the depth generally leads to more label leakage in ID2Graph. Still, it's worth noting that going too deep can sometimes worsen attack results. As the depth of the tree increases, the instance space at the leaves becomes purer but smaller. In other words, fewer classes and samples are assigned in the instance space as the tree splits. ID2Graph relies on the assumption that the data samples within the instance space have similar class labels. While fewer classes can be beneficial, having too small of an instance space can be detrimental. Extremely, if all the leaves contain only one sample, it is hard to extract useful relationship information.

\begin{figure}
  \begin{minipage}[b]{0.48\linewidth}
    \centering
    \includegraphics[width=\linewidth]{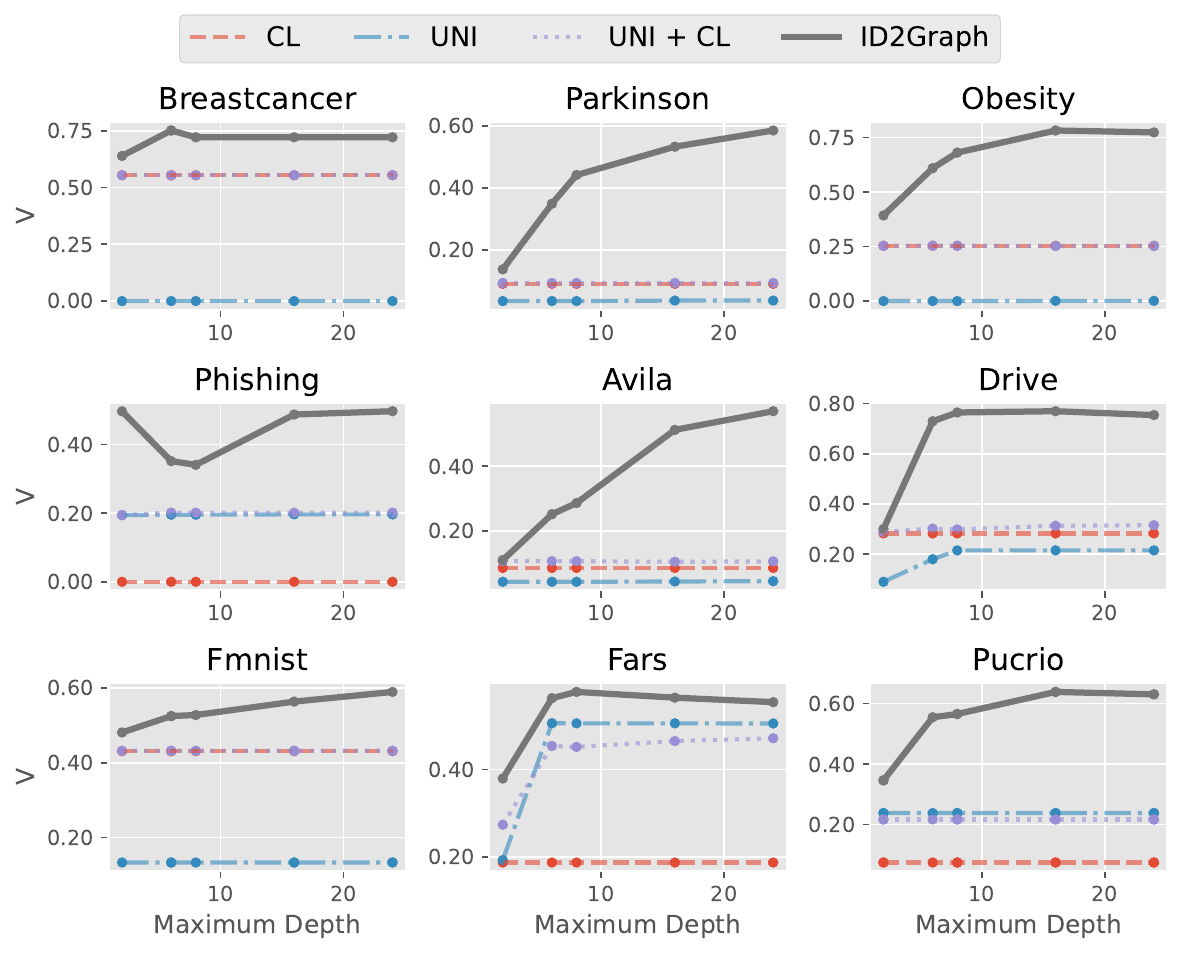}
    \subcaption{Random Forest}
  \end{minipage}
  \hspace{0.002\columnwidth}
  \begin{minipage}[b]{0.48\linewidth}
    \centering
    \includegraphics[width=\linewidth]{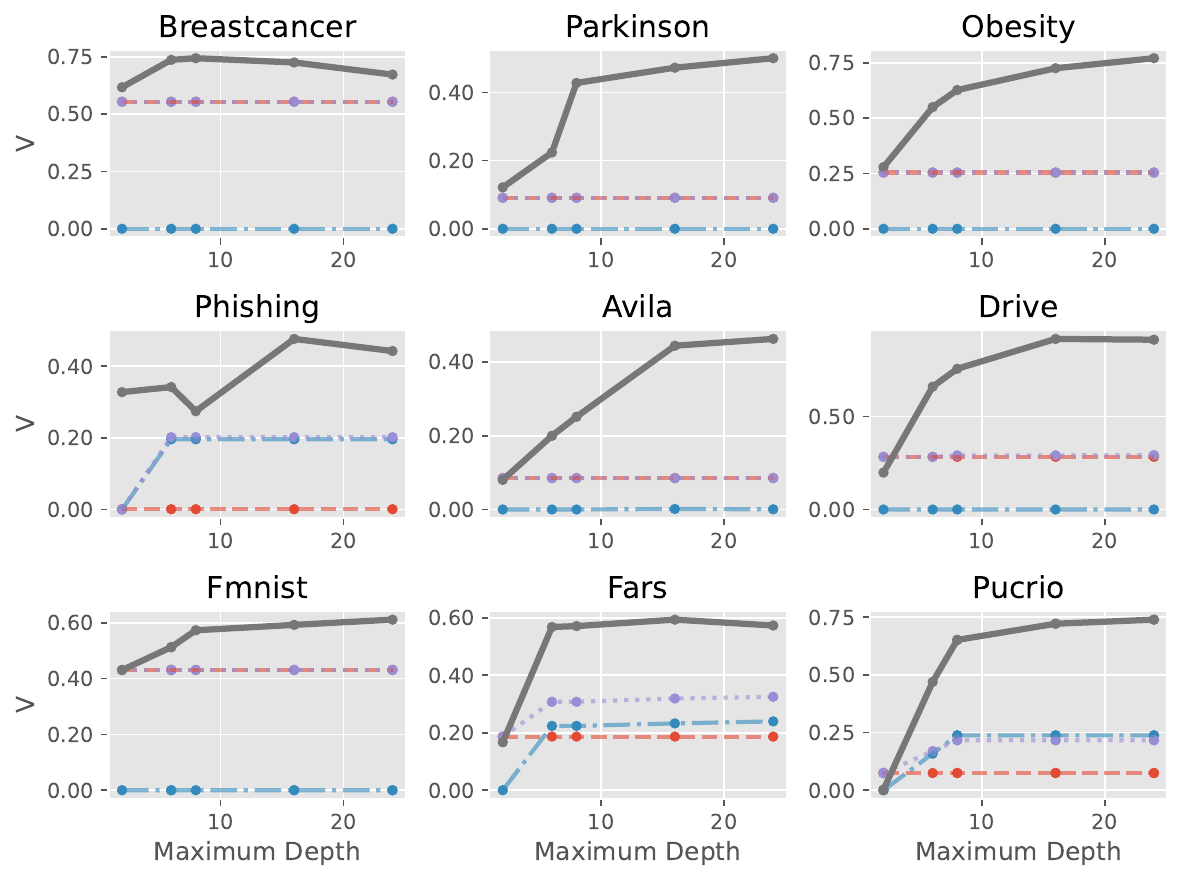}
    \subcaption{XGBoost}
  \end{minipage}
  \caption{Impact of tree depth on attack performance. Deeper depth improves the attack performance of ID2Graph in many cases, but too deep depth can be harmful in some cases. We use different scales for y-axes for better viewing.}
  \label{fig:depth}
\end{figure}

\paragraph{Impact of Number of Trees}

We also investigate the impact of the number of trees on attack performance. As shown in Fig.~\ref{fig:round}, ID2Graph outperforms other baselines in most settings. Fig.~\ref{fig:round} also reveals that while increasing the number of trees generally improves the attack performance, too many trees decrease attack performance, especially when attacking XGBoost. This is compatible with a prior work~\cite{cheng2021secureboost}, which finds that the latter trees in the forest have relatively less information about training labels. 

\begin{figure}
  \begin{minipage}[b]{0.48\linewidth}
    \centering
    \includegraphics[width=\linewidth]{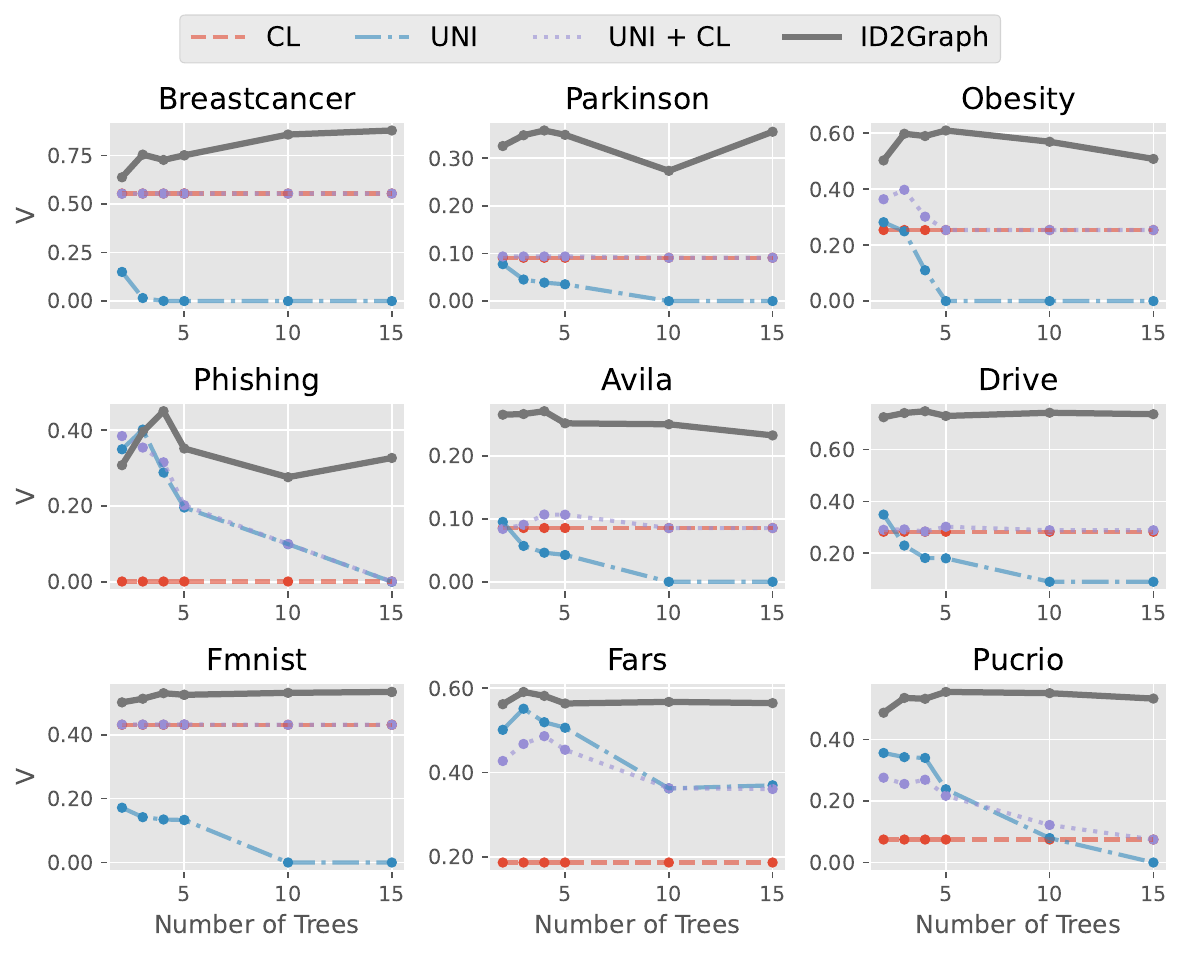}
    \subcaption{Random Forest}
  \end{minipage}
  \hspace{0.002\columnwidth}
  \begin{minipage}[b]{0.48\linewidth}
    \centering
    \includegraphics[width=\linewidth]{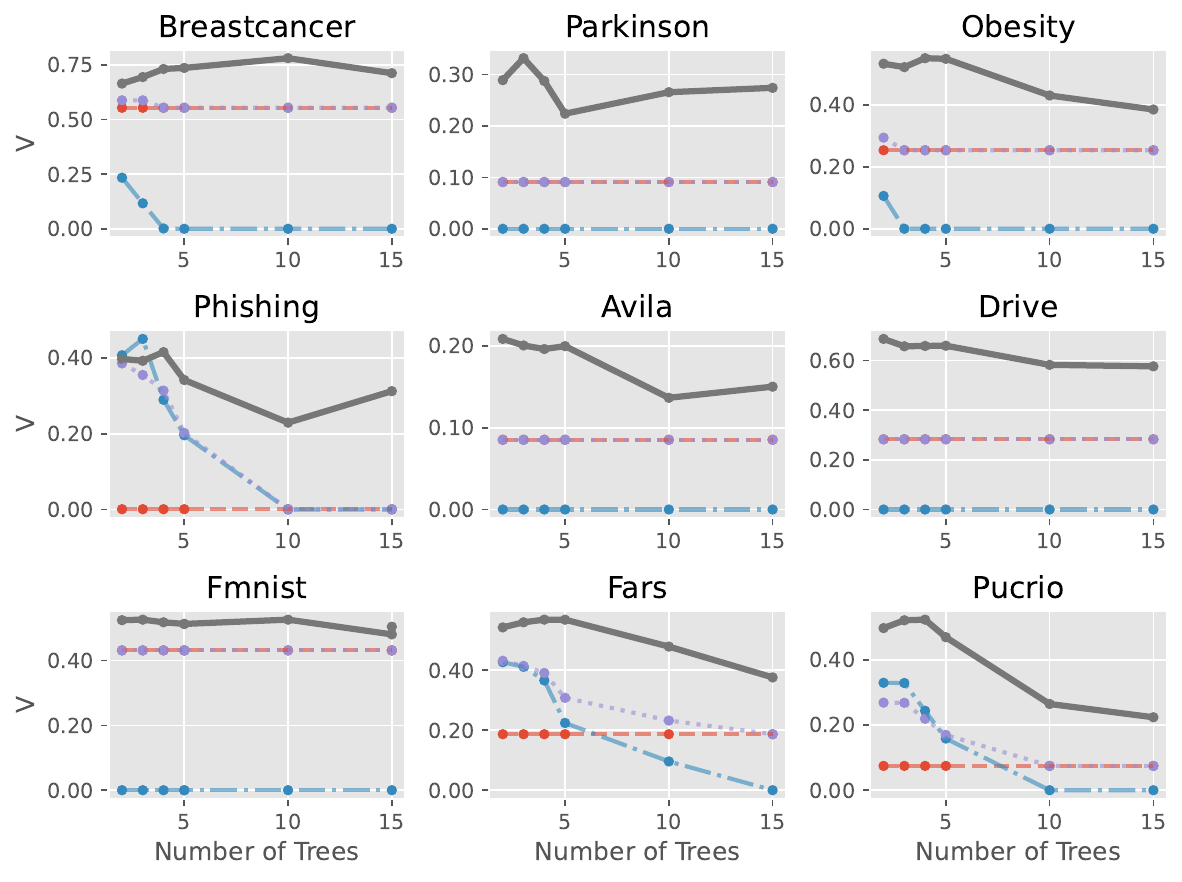}
    \subcaption{XGBoost}
  \end{minipage}
  \caption{Impact of the number of trees. Increasing the number of trees has a positive influence on the performance of the attack against Random Forest. For XGBoost, the larger number of trees does not always to lead better attack performance since the amount of information about the label gradually decreases as the training progresses. We use different scales for y-axes for better viewing.}
  \label{fig:round}
\end{figure}

\subsection{Defense Results}

\paragraph{Main Results}

Fig.~\ref{fig:results:defense} shows the AUC of the trained VFL model on the test dataset (x-axis) and the attack performance (y-axis) with Reduced-Leakage, LP-2ST, Grafting-LDP, and ID-LMID. A defense is considered ideal when its result is located at the bottom right of the figure, indicating high performance on the main VFL task and a lower success rate for the attack. For LP-2ST, Grafting-LDP, and ID-LMID, smaller dots indicate a lower privacy budget ($\epsilon$ or $\xi$), leading to more robust defense at the cost of higher utility loss (therefore appearing on the left bottom of the figure). ID-LMID yields a better trade-off between privacy and utility compared to other methods in most settings. This result suggests that ID-LMID can help find well-fitted tree structures that do not excessively utilize the features of passive parties, which prevents the adversary from obtaining enough information to infer the training labels. Grafting-LDP also significantly improves the utility of LP-2ST while reducing the amount of label leakage. Note that the V-measures of Grafting-LDP and LP-2ST are the same since applying Grafting-LDP does not change the trained model accessible to the attacker.

\begin{figure}[!th]
    \begin{minipage}[b]{0.48\linewidth}
      \centering
      \includegraphics[width=\linewidth]{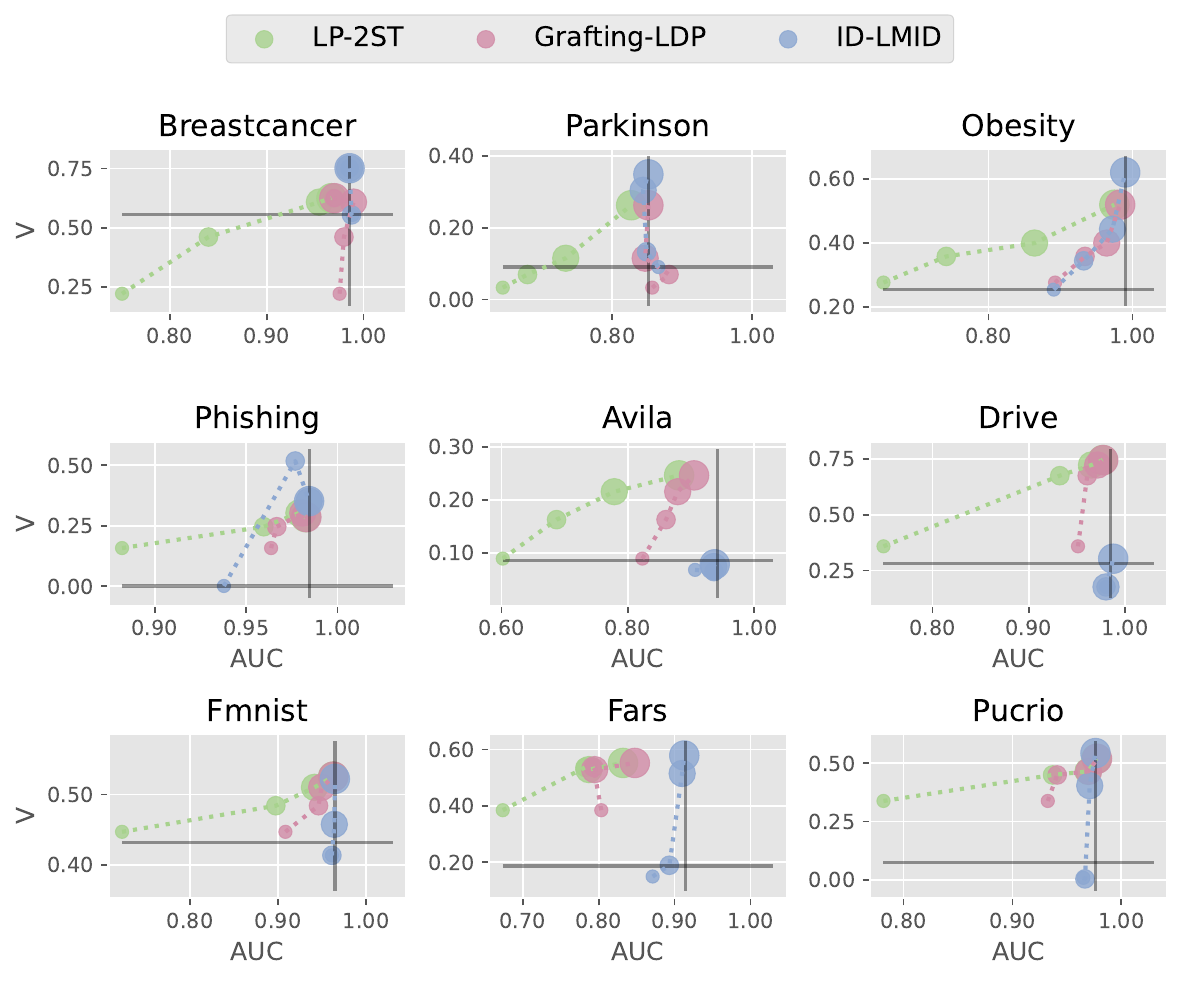}
      \subcaption{Random Forest}
    \end{minipage}
    \hspace{0.01\columnwidth}
    \begin{minipage}[b]{0.48\linewidth}
      \centering
  \includegraphics[width=\linewidth]{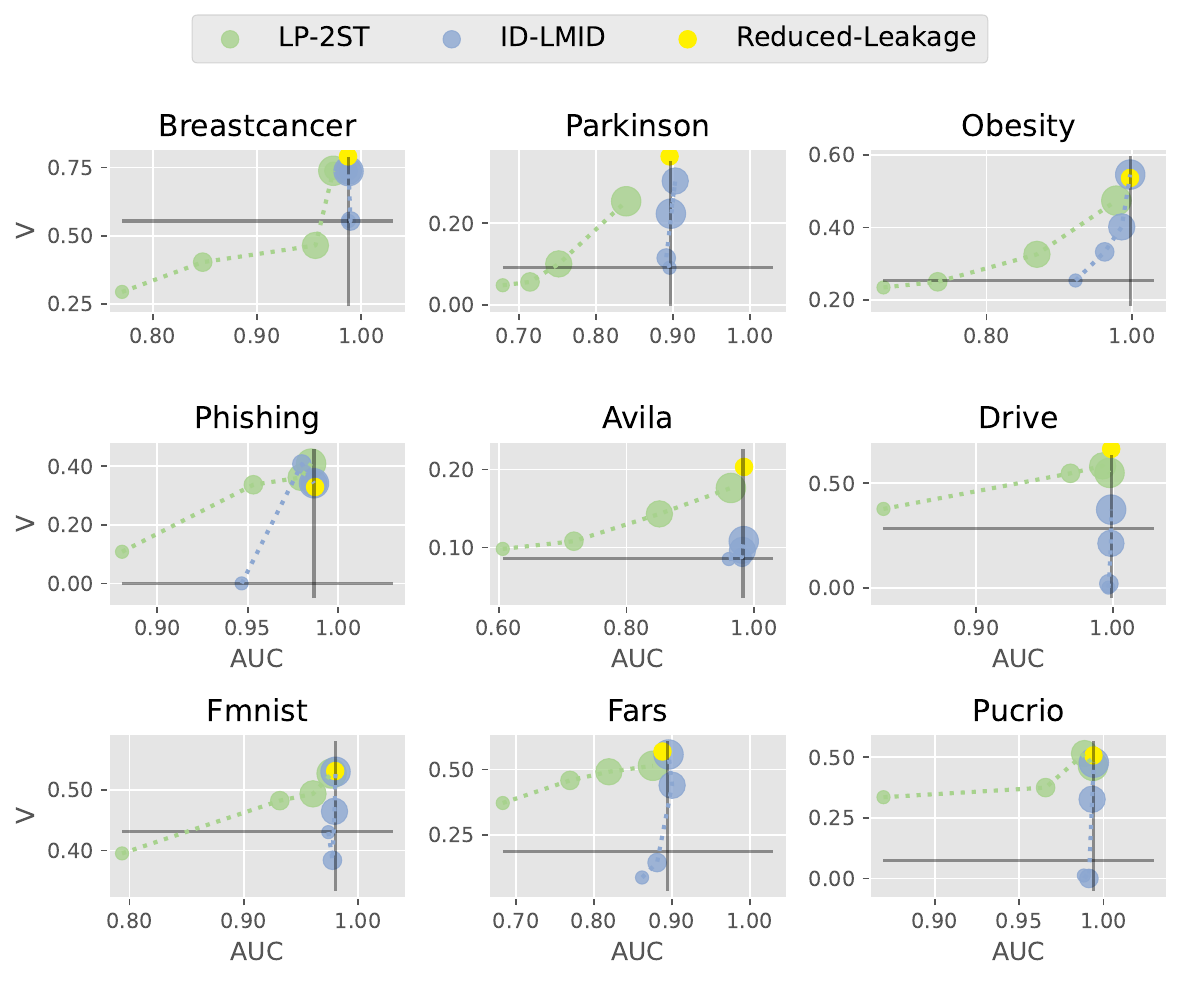}
  \subcaption{XGBoost}
    \end{minipage}

\caption{Defense results against ID2Graph attack: the AUC of the trained VFL model on the test dataset (x-axis) and the attack performance of ID2Graph (y-axis) with Reduced-Leakage, LP-2ST, and ID-LMID defenses. The marker size is proportional to the privacy budget: $\epsilon$ for LP-MST and $\xi$ for ID-LMID. The horizontal black line is the CL results, and the vertical black line shows the AUC of the main task without any defense. ID-LMID achieves the best privacy-utility trade-off. We use different scales for x/y-axes for better viewing.}
  \label{fig:results:defense}
\end{figure}


Fig.~\ref{fig:defense-others} also shows the AUC of the trained model on the test dataset (x-axis) and the attack performance (y-axis) with Reduced-Leakage, LP-2ST, Grafting-LDP and ID-LMID defenses against baseline attacks. Similar to the result of ID2Graph, our defense methods yield better privacy-utility tradeoffs compared to other existing defenses. 


\begin{figure}[!th]
  \begin{minipage}[b]{0.48\linewidth}
    \centering
    \includegraphics[width=\linewidth]{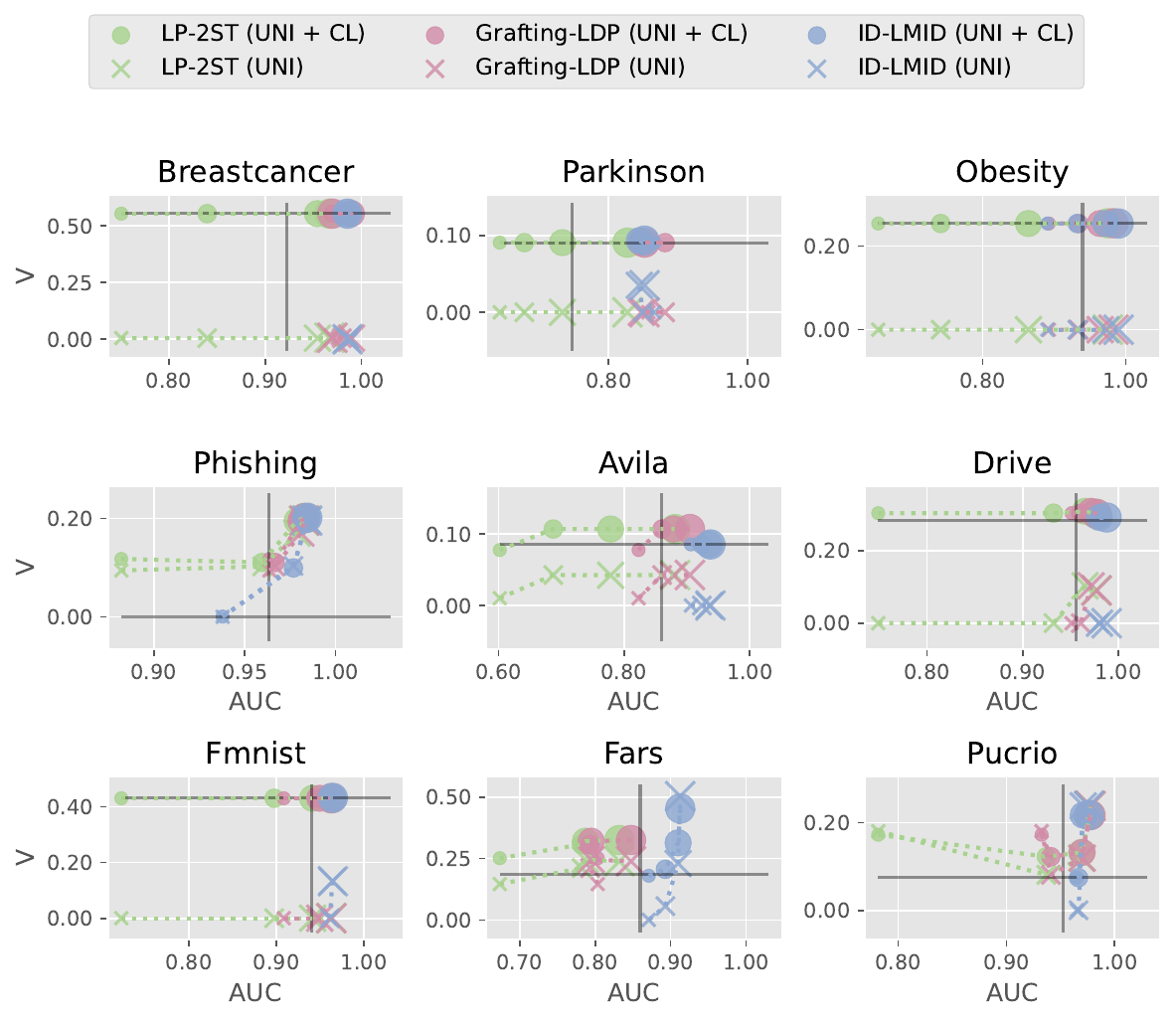}
    \subcaption{Random Forest}
  \end{minipage}
  \hspace{0.002\columnwidth}
  \begin{minipage}[b]{0.48\linewidth}
    \centering
    \includegraphics[width=\linewidth]{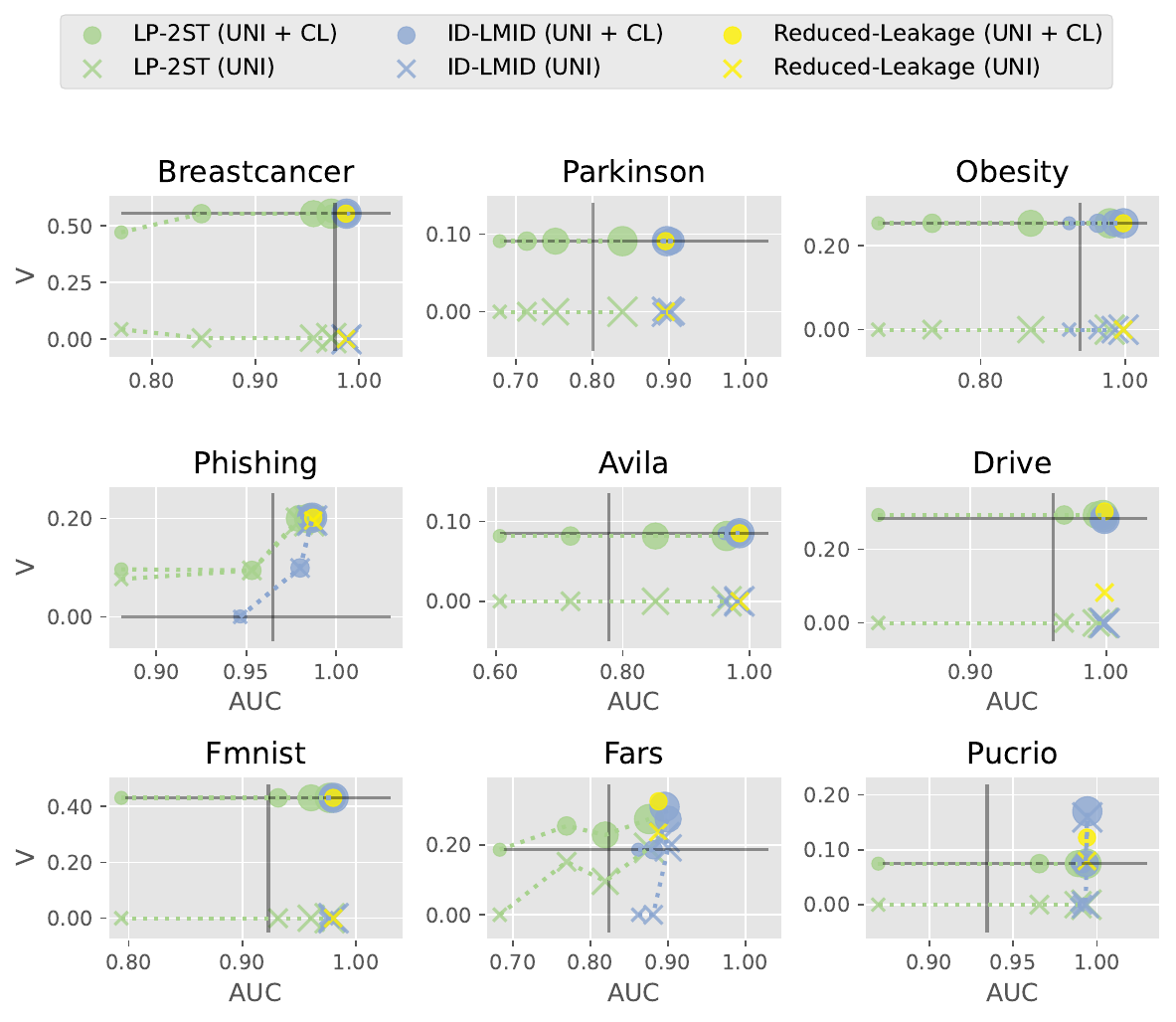}
    \subcaption{XGBoost}
  \end{minipage}
  \caption{Defense results against UNI and UNI + CL. The format is the same as Fig.~\ref{fig:results:defense}. ID-LMID achieves the best privacy-utility trade-off.}
  \label{fig:defense-others}
\end{figure}

\paragraph{Impact of Feature Partition}

The performance of ID-LMID and Grafting-LDP also depends on the availability of features at the active party since they require training with only the features owned by the active party. Fig.~\ref{fig:results:defense-zero} shows the results when the active party has no features, while all other settings are kept the same as Fig~4. This situation is the hardest one for the active party. Reduced-Leakage defense is not applicable in this scenario, and the number of stages of LP-MST is 1. While the AUC of ID-LMID with the same $\xi$ is generally lower compared to Fig.~4, ID-LMID and Grafting-LDP still achieve better utility-privacy tradeoffs than LP-MST.

\begin{figure}[!th]
    \begin{minipage}[b]{0.48\linewidth}
      \centering
      \includegraphics[width=\linewidth]{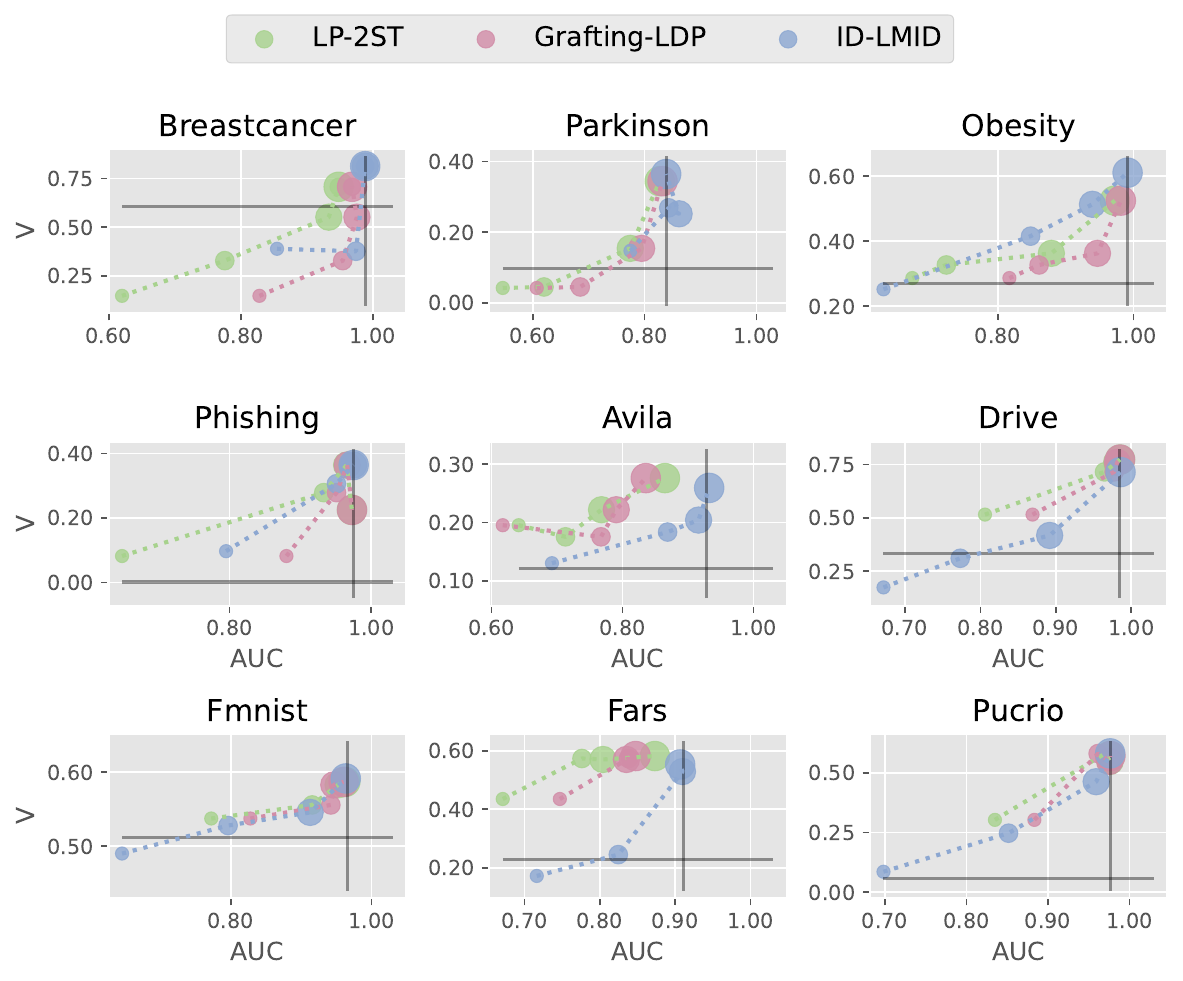}
      \subcaption{Random Forest}
    \end{minipage}
    \hspace{0.002\columnwidth}
    \begin{minipage}[b]{0.48\linewidth}
      \centering
  \includegraphics[width=\linewidth]{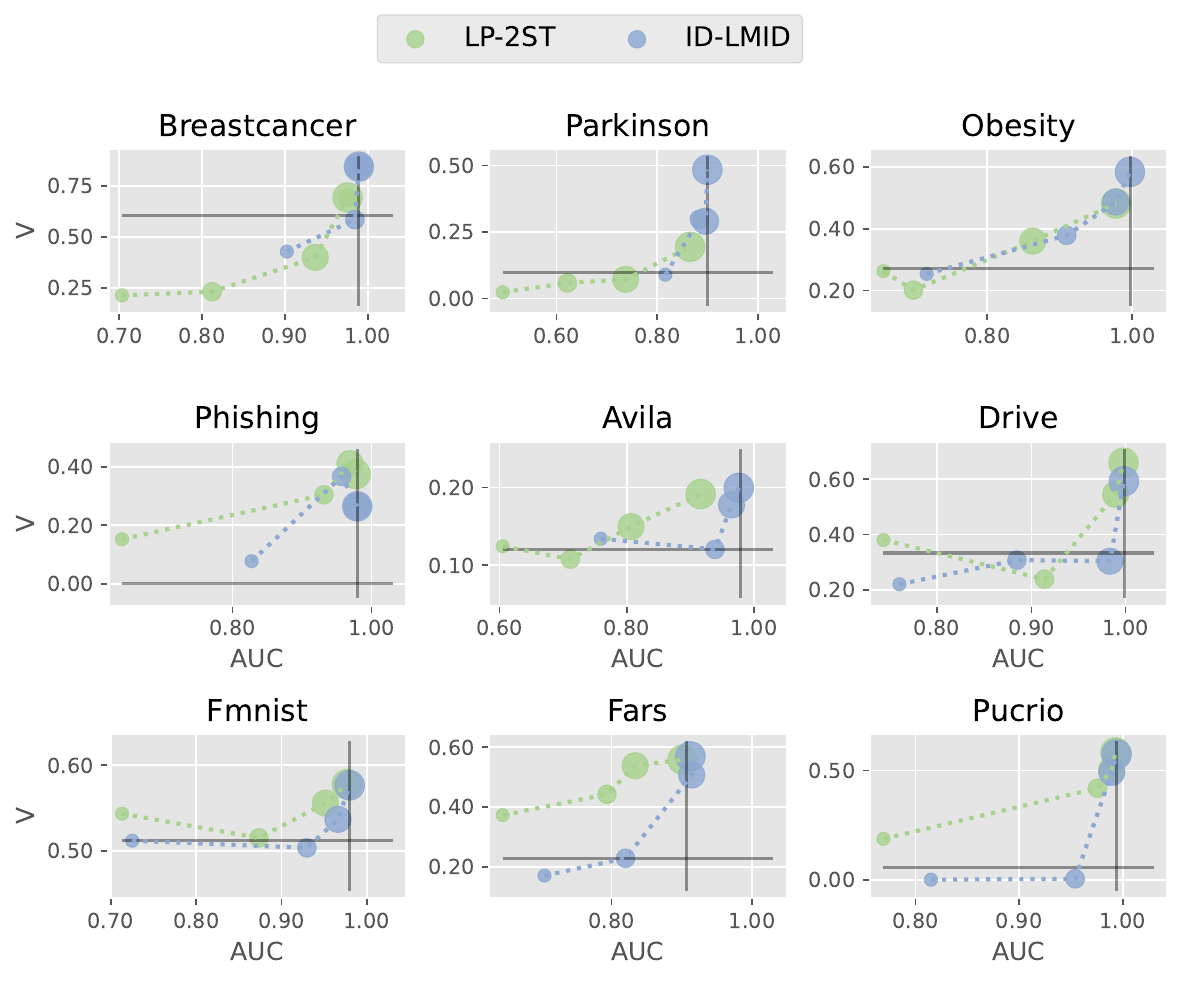}
  \subcaption{XGBoost}
    \end{minipage}

\caption{Defense results against ID2Graph when the active party does not have any features. Our defense methods still allow the active party to protect the training labels with better privacy-utility trade-offs}
  \label{fig:results:defense-zero}
\end{figure}

\paragraph{Communication Costs}

Table~\ref{tab:com-cost} illustrates the comparison of the rate of increase in the number of communicated ciphertexts due to applying each defense. The values are averaged across all datasets. Based on~\cite{cheng2021secureboost,yao2022efficient}, we assume that Random Forest communicates the encrypted labels and their summations within each instance space, while XGBoost communicates the encrypted gradient/hessian and their summations within each instance space. All defenses do not result in excessive additional communication. Reduced-Leakage lowers communication costs by excluding passive parties from the initial tree training. ID-LMID, with a lower $\xi$, also reduces communication by excluding passive parties from training more nodes, thus reducing the number of communicated ciphertexts. In contrast to XGBoost, ID-LMID in Random Forest, as discussed in Section~\ref{subseq:defense-vs}, doesn't require extra communication for evaluating purities, resulting in lower overall communication. Furthermore, because trees trained on noisy labels tend to be larger due to the difficulty of fitting, they lead to higher communication costs. Lastly, Grafting-LDP has an equivalent communication cost to ID-LMID.

\begin{table}[!th]
\centering
\caption{Rate of increase in the number of communicated ciphertexts compared to not applying any defense.}
\label{tab:com-cost}
\begin{tabular}{@{}c|c|c||c@{}}
\toprule
Model & Defense                                                                        & Parameter & \begin{tabular}[c]{@{}c@{}}Rate of\\ Increase\end{tabular} \\ \midrule
\multirow{4}{*}{RadomForest} & \multirow{2}{*}{ID-LMID} & $\xi=0.5$  & 0.364 \\
                             &                          & $\xi=2.0$  & 0.820 \\ \cmidrule(l){2-4}
      & \multirow{2}{*}{\begin{tabular}[c]{@{}c@{}}LP-MST $/$ \\ Grafting-LDP\end{tabular}} & $\epsilon=0.5$   & 1.16                                                       \\
                             &                          & $\epsilon=2.0$ & 1.14  \\ \midrule
\multirow{5}{*}{XGBoost}     & \multirow{2}{*}{ID-LMID} & $\xi=0.5$  & 0.682 \\
                             &                          & $\xi=2.0$  & 1.17  \\ \cmidrule(l){2-4}
                             & \multirow{2}{*}{LP-MST}  & $\epsilon=0.5$ & 1.18  \\
                             &                          & $\epsilon=2.0$ & 1.17  \\ \cmidrule(l){2-4}
                             & Reduced-Leakage                       &         & 0.798 \\ \bottomrule
\end{tabular}
\end{table}



\section{Related Work}
\label{sec:rw}

\paragraph{Tree-based Vertical FL}

Federated learning (FL) is a technique that enables training models on decentralized data sources without sharing local raw data. Vertical federated learning (VFL) is one type of FL where each party owns a vertically partitioned dataset. Tree-based Vertical FL (T-VFL) has been actively studied due to its efficiency and practicality. Whereas classic works~\cite{10.1007/11562382_75,du2002building,vaidya2008privacy,gangrade2009building} mostly focus on securely training a single decision tree on a vertically federated dataset, recent works propose algorithms for privacy-preserving tree ensembles, including bagging~\cite{liu2020federated,hou2021verifiable,XU2023237,9514457,yao2022efficient} and boosting ~\cite{cheng2021secureboost,10.1145/3448016.3457241,chen2021secureboost+,tian2020federboost,zhu2021pivodl,wang2022feverless}.

\paragraph{Label Inference Attack in VFL}

Label inference attack in VFL is typically conducted at a passive party of a VFL system in order to infer training labels held by the active party~\cite{Liu2022VFLsurvey}. Most existing works only apply to VFL systems trained on logistic regression~\cite{tan2022residue} or neural networks~\cite{li2021label,fu2022label,sun2022label,kariyappa2021gradient,qiu2022your,wainakh2022user,9369498} models, where T-VFL systems are much less studied. For T-VFL, \citep{wu2020privacy} suggests that the attacker can infer each sample's label if each node's weight is obtainable. However, this is infeasible in many protocols, such as Secureboost where the weight of each node is not exposed to passive parties. \citep{cheng2021secureboost} finds that the instance spaces, which are not concealed from passive parties in many existing studies~\cite{cheng2021secureboost,10.1007/11562382_75,liu2020federated,hou2021verifiable,10.1145/3448016.3457241,chen2021secureboost+,XU2023237,9514457,yao2022efficient,tian2020federboost,zhu2021pivodl,wang2022feverless}, might cause label leakage, yet how accurately the attacker can infer the labels from these instance spaces are not well-studied.

\paragraph{Defense against Label Leakage in VFL}

Current defense strategies can be broadly categorized into non-cryptographic and cryptographic approaches. Cryptographic approaches protect intermediate information, including the instance space, using Multi-party Computation (MPC)~\cite{hoogh2014practical,adams2021privacy,cryptoeprint:2020/1130,deforth2021xorboost} or Homomorphic Encryption (HE) ~\cite{wu2020privacy,fang2021large}, but the communication and computation costs are too high for realistic situations ~\cite{wu2020privacy,fang2021large}. Non-cryptographic approaches aim to prevent information leakage by introducing constraints or noise into the training process. Reduced-Leakage for SecureBoost~\cite{cheng2021secureboost} trains the first tree without using passive parties, thus limiting their access to label information. However, this approach relies on the strong assumption that the training dataset is large enough and the model depth is shallow enough. Differential privacy (DP) rigorously quantifies information leakage from statistical algorithms~\cite{dwork2014algorithmic,alvim2011differential}. Label DP ~\cite{ghazi2021deep,beimel2013private,pmlr-v19-chaudhuri11a} is proposed to tackle the situation where only labels are sensitive information. One DP-based practical algorithm that prevents label leakage from tree-based models is LP-MST~\cite{ghazi2021deep}. However, differential privacy often sacrifices utility to prevent privacy attacks sufficiently~\cite{wang2021improving}. Mutual information-based defense limits the dependency between two variables to prevent the adversary from estimating one variable from the other~\cite{wang2021improving,farokhi2020modelling}, while applicability on label leakage in T-VFL has not been studied before.

\section{Conclusion}
\label{seq:conclusion}

This work explores the vulnerability of tree-based VFL (T-VFL) to label inference attacks and demonstrates that instance spaces exchanged in a typical T-VFL system can be exploited to infer sensitive training labels via ID2Graph attack. To counteract label leakage, we propose a mutual information-based defense, ID-LMID, and a defense based on label differential privacy, Grafting-LDP. Experiments on diverse datasets illustrate the significant risk of label leakage from the instance space, as well as the effectiveness of Grafting-LDP and ID-LMID compared with other existing defenses. Future work for the attack part might include better community detection and clustering methods, such as community detection with node attributes and other weighting strategies for the adjacency matrix. To enhance the defense algorithms, we will also conduct research on improving Grafting-LDP to be compatible with boosting-based methods and tighter upper bounds of mutual information between instance space ad labels. We hope our study stimulates a re-evaluation of data safety for T-VFL schemes and inspires future work on defense strategies for T-VFL.

\bibliographystyle{unsrt}
\bibliography{ref}

\end{document}